\documentclass{new_tlp}
\usepackage{amsmath}
\usepackage{amsfonts}
 \usepackage{newfloat}

\usepackage{graphicx}
\usepackage{epstopdf}
\usepackage{subfigure}
\DeclareFloatingEnvironment[name=Operator]{myoperator}

\title{Lifted Variable Elimination for Probabilistic Logic Programming}
\author[E. Bellodi, E. Lamma, F. Riguzzi, V. Santos Costa and R. Zese] 
{ELENA BELLODI$^1$, EVELINA LAMMA$^1$,
FABRIZIO RIGUZZI$^2$ \and 
VITOR SANTOS COSTA$^3$, RICCARDO ZESE$^1$ \\
$^1$ Dipartimento di Ingegneria -- Universit\`a di Ferrara\\
Via Saragat 1, 44122, Ferrara, Italy  \\
$^2$ Dipartimento di Matematica e Informatica -- Universit\`a di Ferrara\\
Via Saragat 1, 44122, Ferrara, Italy  \\
$^3$ CRACS and DCC-FCUP -- Universidade do Porto\\
Rua do Campo Alegre, 1021/1055, 4169-007 Porto, Portugal\\
\email{name.surname@unife.it,vsc@dcc.fc.up.pt}}

\pagerange{\pageref{firstpage}--\pageref{lastpage}}
\doi{}

\newtheorem{theorem}{Theorem}
\newtheorem{example}{Example}
\newtheorem{definition}{Definition}

\begin{document}

\label{firstpage}

\maketitle

% \noindent
% \begin{center}
% {\bf Note:} To appear in Theory and Practice of Logic Programming (TPLP).
% }

\begin{abstract}
Lifted inference has been proposed for various probabilistic logical frameworks in order to compute the 
probability of queries in a time that depends on the size of the domains of the random variables rather than
the number of instances.
Even if various authors have underlined its importance for probabilistic logic programming (PLP), lifted inference
has been applied up to now only to relational languages outside of logic programming. In this paper we adapt Generalized
Counting First Order Variable Elimination (GC-FOVE) to the problem of computing the probability of queries to 
probabilistic logic programs under the distribution semantics. In particular, we extend the Prolog Factor Language
(PFL) to include two new types of factors that are needed for representing ProbLog programs.
These factors take into account the
existing causal independence relationships among random variables and are managed by the extension to variable
elimination proposed by Zhang and Poole for dealing with convergent variables and heterogeneous factors.
Two new operators are added to GC-FOVE for treating heterogeneous factors.
The resulting algorithm, called LP$^2$ for Lifted Probabilistic Logic Programming, has been implemented by modifying the PFL implementation of GC-FOVE and tested on three benchmarks for 
lifted inference. A comparison with PITA and ProbLog2 shows the potential of the approach. 
\end{abstract}

\begin{keywords}
Probabilistic Logic Programming, Lifted Inference, Variable Elimination, Distribution Semantics, ProbLog, Statistical Relational Artificial Intelligence
\end{keywords}

%\tableofcontents

\section{Introduction}

Over the last years, there has been increasing interest in
models that combine first-order logic and probability, both for domain modeling under uncertainty, and for efficiently performing inference and learning~\cite{Getoor+al:book07,DBLP:conf/ilp/2008p}.
Probabilistic Logic Programming (PLP) has recently received an increasing attention for its ability to incorporate
probability in logic programming. Among the various proposals, the one based on the distribution semantics~\cite{DBLP:conf/iclp/Sato95}
has gained popularity as the basis of languages such as Probabilistic Horn Abduction~\cite{DBLP:journals/ai/Poole93}, PRISM~\cite{DBLP:conf/iclp/Sato95}, Independent
Choice Logic~\cite{Poo97-ArtInt-IJ}, Logic Programs with Annotated Disjunctions~\cite{VenVer04-ICLP04-IC}, and ProbLog~\cite{DBLP:conf/ijcai/RaedtKT07}.

Nonetheless, research in Probabilistic Logic Languages has made it very clear that it is crucial to design models that can support efficient inference, while preserving intensional, and declarative modeling. 
Lifted inference~\cite{Poole:2003,DBLP:conf/ijcai/BrazAR05,DBLP:conf/aaai/MilchZKHK08,DBLP:conf/ijcai/BroeckTMDR11} is one of the major advances in this respect. The idea is to take advantage of the regularities in structured models to decrease the number of operations. 
Originally, the idea was proposed as an extension of variable elimination ({\sc ve} for short in the following). Work on lifting {\sc ve} started with \cite{Poole:2003}.
Lifted VE exploits the symmetries present in first-order probabilistic models, so that it can apply the same principles behind {\sc ve} to solve a probabilistic query without grounding the model.

Most work in probabilistic inference compute statistics from a  sum of
products representation, where each element is named a {\em
  factor}. Lifted inference generates templates, named {\em parametric
  factors} or {\em parfactors}, which stand for a set of similar factors found in the inference process, thus delaying as much as possible the use of fully instantiated factors. 

In \cite{DBLP:conf/ilp/GomesC12}, the {\em Prolog Factor Language} (PFL, for short) was proposed as a Prolog extension to support probabilistic reasoning with parfactors. 
PFL exploits the state-of-art algorithm GC-FOVE~\cite{DBLP:journals/jair/TaghipourFDB13}, which redefines the operations described in\cite{DBLP:conf/aaai/MilchZKHK08} to be correct for whatever constraint representation is being used. This decoupling of the lifted inference algorithm from the constraint representation mechanism allows  any constraint language that is closed under these operators to be plugged into the algorithm to obtain an inference system. In fact, the lifted {\sc ve} algorithm of \cite{DBLP:conf/ilp/GomesC12} represents the adaptation of GC-FOVE to the PFL constraints.
%, with the difference that they use a simple tree to represent the constraints instead of the GC-FOVE constraint tree.

In this work, we move further towards exploiting efficient inference via lifted {\sc ve} for PLP Languages under the distribution semantics. To support reasoning compliant with the distribution semantics, we introduce two novel operators (named {\em heterogeneous } lifted multiplication and sum) in the PFL, and modify the GC-FOVE algorithm for computing them. 
We name LP$^2$ (for Lifted Probabilistic Logic Programming) the resulting system. 
An experimental comparison between LP$^2$ and ProbLog2 \cite{problog2} and PITA \cite{RigSwi11-ICLP11-IJ} shows that inference time increases linearly with the number of individuals of the program domain for LP$^2$, rather than exponentially as with ProbLog2 and PITA.

This is an exciting development towards the goal of preserving the declarativeness and conciseness of Probabilistic Logic Languages, while extremely gaining in performances. 

The paper is organized as follows. Section \ref{pre} introduces preliminaries regarding ProbLog,  PFL, Causal Independence Variable Elimination, and GC-FOVE. Section \ref{trans} discusses the translation of ProbLog into the extended PFL. Section \ref{hetop} presents the new operators introduced in GC-FOVE. Section \ref{exp} reports the experiments performed and Section \ref{conc} concludes the paper.

\section{Preliminaries}
\label{pre}
\subsection{ProbLog}
ProbLog \cite{DBLP:conf/ijcai/RaedtKT07} is a Probabilistic Logic Programming (PLP) language. % based on the distribution semantics \cite{DBLP:conf/iclp/Sato95} with the simplest syntax. 
% \paragraph{Syntax}
A ProbLog program consists of a set of \emph{ground probabilistic facts} plus
a definite logic program, i.e.\ a set of rules. A ground probabilistic fact, written $p::f$, is a ground fact $f$ annotated with a number $p$ such that $0 \leq p \leq 1$. 
An atom that unifies with a ground probabilistic fact is called a \textit{probabilistic atom}, while an atom that unifies with the head of some rule in the logic program is called a \textit{derived atom}. %The set of probabilistic atoms must be disjoint from the set of derived atoms and 

If a set of probabilistic facts has the same probability $p$, it can be defined intensionally through the syntax $p::f(X_1,X_2,\ldots,X_n) :- B$, where $f(X_1,X_2,\ldots,X_n)$ is the signature of the set, and $B$ is a conjunction of non-probabilistic goals, as shown in Example \ref{ws_attr_problog}. Such rules are range-restricted: all variables in the head of a rule should also appear in a positive literal in the body.

\begin{example}[Running example]
\label{ws_attr_problog}
Here we present an example inspired by the \textit{workshop attributes} problem of \cite{DBLP:conf/aaai/MilchZKHK08}. The ProbLog program models
the scenario in which a workshop is being organized and a number of people have been invited. \verb|series| indicates 
whether the workshop is successful
enough to start a series of related meetings while
\verb|attends(P)| indicates whether  person \verb|P| will attend the workshop. 
\begin{footnotesize}
\begin{verbatim}
series :- s.
series :- attends(P).
attends(P) :- at(P,A).
0.1::s.
0.3::at(P,A) :- person(P), attribute(A).
\end{verbatim}
\end{footnotesize}
The first two rules define when the workshop becomes a series: either because of its own merits or because people attend. The third rule states that whether a person attends the workshop depends on its attributes (location, date, fame of the organizers, etc).

The probabilistic fact \verb|s| represents the merit of the workshop. The probabilistic fact \verb|at(P,A)| represents whether person \verb|P| attends because of attribute \verb|A|. Notice that the last statement corresponds to a set of ground probabilistic facts, one for each person \verb|P| and attribute \verb|A|.
For brevity we do not show the (non-probabilistic) facts describing \verb|person/1| and \verb|attribute/1| predicates.  \end{example}

%\paragraph{Semantics}

A ProbLog program specifies a probability distribution over normal logic programs.  In this work, we consider the semantics in the case of no function symbols to be restricted to finite programs, and assume all worlds have a \emph{two-valued} well-founded model.

For each ground probabilistic fact $p_i::f_i$, an \textit{atomic choice} specifies whether to include $f$ in a world (with probability $p_i$) or not (with probability $1-p_i$). A \textit{total choice C} is a set of atomic choices, one for each ground probabilistic fact. These choices are assumed to be independent, hence %Formally, a total choice is any subset of the set of all ground probabilistic atoms. Hence, if there are n ground probabilistic atoms then there are 2n total choices. 
%Moreover, we have a probability distribution over these total choices: 
the probability of a total choice is the product of the probabilities of the individual atomic choices, $P(C) = \Pi_i{p_i}$.  A total choice $C$ also identifies a normal logic program $W=F\cup R$ called a \emph{world}, where $F$ is the set of facts to be included according to $C$ and $R$ denotes the rules in the ProbLog program. Let $\mathcal{W}$ be the set of all possible worlds.  The probability of a world is equal to the probability of its total choice. The conditional probability of a query (a ground atom) $Q$ given a world $W$ $P(Q|W)$ is 1 if the $Q$ is true in the well-founded model of $W$ and $0$ otherwise.  The probability of a query can therefore be obtained as $P(Q)=\sum_{W\in \mathcal{W}}P(Q,W)=\sum_{W\in \mathcal{W}}P(Q|W)P(W)= \sum_{W\in \mathcal{W}:W\models Q}P(W)$.

\subsection{The Prolog Factor Language}

Most graphical models provide a concise representation of a joint
distribution by encoding it as a set of factors. The probability of a
set of variables $\mathbf{X}$ taking the value $\mathbf{x}$ can be
 expressed as product of $n$ factors if:
\[
P(\mathbf{X=x})=\frac{\prod_{i=1,\ldots,n}\phi_i(\mathbf{x}_i)}{Z} \]
where $\mathbf{x}_i$ is a sub-vector of $\mathbf{x}$ that depends on
the $i$-th factor and $Z$ is a normalization constant
(i.e. $Z= \sum_{\mathbf{x}}\prod_{i=1,\ldots,n}\phi_i(\mathbf{x}_i)$). 
Bayesian networks are an example where there is a factor for each
variable that is a function of the variable $X_i$ and its parents
$X_j\ldots X_k$, such that $\phi(X_i,X_j\ldots X_k) = P(X_i|X_j\ldots
X_k)$ and $Z=1$. As progress has been made on managing large
networks, it has become clear that often the same factor appears repeatedly in the network, thus
suggesting the use of templates generalizing individual factors, or
\emph{parametric factors}~\cite{DBLP:conf/uai/KisynskiP09}. 

The Prolog Factor Language (PFL) \cite{DBLP:conf/ilp/GomesC12} extends Prolog to support probabilistic reasoning with parametric factors or \emph{parfactors}.
The PFL syntax for a factor is $Type\ F\ ;\ \phi\ ;\ C$.  $Type$ refers to the type of the network over which the parfactor is defined ($bayes$ for directed
networks or $markov$ for undirected ones); $F$ is a sequence of Prolog terms that define sets
of random variables under the constraints in $C$. The set of all logical variables in $F$ is named $L$.  $C$ is a list of Prolog goals that impose bindings on the logical variables in $L$ (the successful substitutions for the goals in $C$ are the valid values for the variables in $L$). $\phi$  is the table defining the factor in the form of a list of real values.
By default all random variables are boolean but a different domain may be defined.
An example of a factor is
\verb|series,attends(P);[0.51,0.49,0.49,0.51];[person(P)]|:
it has the Boolean random variables \verb|series| and \verb|attends(P)| as arguments, \verb|[0.51,0.49,0.49,0.51]| as table and
\verb|[person(P)]| as constraints. 
The semantics of a PFL program is given by the set of factors obtained by grounding parfactors: each parfactor stands for the
set of its grounding obtained by replacing variables of $L$ with the values allowed by the constraints in $C$. The set of ground factors define a factorization of the joint probability distribution over all random variables.

\begin{example}[PFL Program]
\label{ws_attr_pfl}
A version of the \textit{workshop attributes} problem presented in Example \ref{ws_attr_problog} can be modeled by a PFL program such as
{\footnotesize
\begin{verbatim}
bayes attends(P), at(P,A) ; [0.7, 0.3, 0.3, 0.7] ; [person(P),attribute(A)].
bayes series, attends(P) ; [0.51, 0.49, 0.49, 0.51] ; [person(P)].
\end{verbatim}}
\end{example}

\subsection{Variable Elimination and Causal Independence}
Quite often we want to find out the probability distribution of a set of
random variables $\mathbf{X}$ given that we know the values, or have evidence
$\mathbf{y}$, on a set of variables $\mathbf{Y}$, where $\mathbf{X}$
is often a single variable $X$. Variable Elimination ({\sc
  ve})~\cite{DBLP:journals/jair/ZhangP96} is an algorithm for
computing this \emph{posterior} probability in factorized joint probability
distributions. The key idea is to  eliminate the
random variables from a set of factors one by one until only the query
variable $X$ remains. To do so {\sc ve} eliminates a variable $V$ by first
multiplying all the factors that include $V$ into a single factor; $V$
can then be discarded through summing it out from the newly
constructed factor. % The operation repeats until we have a single
% factor with the variable $X$.

More formally, suppose $\phi_1(X_1\ldots X_i, Y_1\ldots Y_j)$ and $\phi_2(Y_1\ldots
Y_j, Z_1 \ldots Z_k)$ are factors.  % You can see a factor as a table
% that maps each value of the vector of variables that are its arguments
% to a real number. 
The product $(\phi_1\times \phi_2)(x_1 \ldots x_i, y_1\ldots y_j,
z_1\ldots z_k)$ is simply $\phi_1(x_1 \ldots x_i, y_1\ldots
y_j,)\times  \phi_2(y_1$ $\ldots y_j, z_1\ldots z_k)$ for every value of
$x_1 \ldots x_i, y_1\ldots y_j, z_1\ldots z_k$.  To eliminate a variable $X_1$ from the factors $\phi(x_1\ldots
x_i)$ one observes that the cases for $X_1$ are mutually exclusive, thus
$\phi'(x_2\ldots x_i$) is simply $(\sum_{x_1}\phi)(x_2\ldots x_i)=
\phi(\alpha_1,x_2\ldots x_i)+ \ldots+\phi(\alpha_m,x_2\ldots x_i)$,
where $\alpha_1,\ldots,\alpha_m$ are the possible values of $X_1$.

The full {\sc ve} algorithm  takes as input a set
of factors $\mathcal{F}$, an elimination order $\rho$, a set of query
variables $\mathbf{X}$ and a list $\mathbf{y}$ of observed
values. First, it sets the observed variables in all factors to their
corresponding observed values. Then it repeatedly selects the first
variable $Z$ from the elimination order $\rho$ and it calls {\sc
  sum-out} on $\mathcal{F}$ and $Z$, until $\rho$ becomes empty. In
the final step, it multiplies together the factors of $\mathcal{F}$
obtaining a new factor $\gamma$ that is normalized as
$\gamma(x)/\sum_{x'}\gamma(x')$ to give the posterior probability.

\paragraph{Noisy OR-Gates}

Bayesian networks take advantage of conditional independence between
variables to reduce the size of the representation. \emph{Causal
  independence} \cite{DBLP:journals/jair/ZhangP96} goes one step
further and looks at independence conditioned on \emph{values} of the
random variables. One important example is the \emph{noisy OR-gate},
where we have a Boolean variable $X$ with parents $\mathbf{Y}$, and
ideally $X$ should be true if any of the $Y_i$ is true.  In practice,
each parent $Y_i$ has a noisy inhibitor that independently blocks or
activates $Y_i$, so $X$ is true if either \textbf{any} of the causes
$Y_i$ holds true \emph{and} is not inhibited.
%, or if one $Y_i$ is false
%but the inhibitor was activated.

%{\centering\begin{tabular}{cccr}
%Y_1 & Y_2 & X\\\hline
%f&f&f& 0.1\\\hline
%f&f&t& 0.9\\\hline
%f&t&f& 0.3\\\hline
%f&t&t& 0.7\\\hline
%t&f&f& 0.3\\\hline
%t&f&t& 0.7\\\hline
%t&t&f& 0.51\\\hline
%t&t&t& 0.49\\\hline
%\end{tabular}}
% {\centering\begin{tabular}{|l|r|r|r|r|r|r|r|r|}
% \hline
% \phi(Y_1, Y_2 , X)& \textit{fff}& \textit{fft}& \textit{ftf}& \textit{ftt}& \textit{ff}& \textit{tft}& \textit{ttf}& \textit{ttt}\\\hline
% & 0.02 & 0.98 & 0.04 & 0.96 & 0.24 & 0.76 & 0.42 & 0.58 \\\hline
% \end{tabular}}
% The factor $\phi$ can be expressed as a combination of two simpler factors of two variables $\psi$ and $\gamma$  
% $$
% \begin{array}{cc}
% \begin{array}{|l|r|r|r|r|}
% \hline
% \psi(Y_1, X')& \textit{ff}& \textit{ft}& \textit{tf}& \textit{tt}\\\hline
% & 0.1 & 0.9 & 0.3 & 0.7  \\\hline
% \end{array}&
% \begin{array}{|l|r|r|r|r|}
% \hline
% \gamma(Y_2, X'')& \textit{ff}& \textit{ft}& \textit{tf}& \textit{tt}\\\hline
% & 0.2 & 0.8 & 0.4 & 0.6  \\\hline
% \end{array}
% \end{tabular}}
A noisy OR can be expressed as a factor $\phi$. In fact, it can be
also expressed as a combination of factors by introducing intermediate
variables that represent the effect of each cause \emph{given the
  inhibitor}. For example, if $X$ has two causes $Y_1$ and $Y_2$, we
can introduce a variable $X'$ to account for the effect of $Y_1$ and
$X''$ for $Y_2$, and the factor $\phi(Y_1, Y_2 , X)$ can be expressed
as
\begin{equation}\label{comb}
\phi(y_1,y_2,x)=\sum_{x' \vee x''=x}\psi(y_1,x')\gamma(y_2,x'')
\end{equation}
where the summation is over all values $x'$ and $x''$ of $X'$ and $X''$ whose disjunction is equal to $x$.
%The factors $\psi$, and $\gamma$ model the inhibitors.
%and $\varphi(X',X'',X)$  models the logical OR
%function. % The factor
% $\phi$ is known as \textit{noisy OR-gate}.  In general, causal
% independence can be defined for binary operators other than OR
% provided they are commutative and associative. In this paper we
% however restrict attention only to the OR operation and to Boolean
% variables.
The $X$ variable is called \textit{convergent} as it is where
independent contributions from different sources are collected and
combined.  Non-convergent variables will  be called
\textit{regular variables}.
%  Representing factors such as $\phi$ with
% $\psi$ and $\gamma$ is advantageous when the number of parents grows
% large, as the combined size of the component factors grow linearly,
% instead of exponentially.

The noisy OR thus allows for a $O(n)$ representation of a conditional
probability table with $n$ parents. Unfortunately,  straightforward use of 
  {\sc ve} for inference would lead to construct $O(2^n)$ tables. A
modified algorithm, called {\sc
  ve1}~\cite{DBLP:journals/jair/ZhangP96}, % represents factors such as
% $\psi$ and $\gamma$ as functions of $Y_1,X$ and $Y_2,X$,
% respectively. The factors are 
combines factors through a new operator
$\otimes$, that generalizes formula (\ref{comb}) as follows.  Let
$\phi$ and $\psi$ be two factors that share convergent variables
$E_1\ldots E_k$, let $\mathbf{A}$ be the list of regular variables
that appear in both $\phi$ and $\psi$, let $\mathbf{B_1}$ ($\mathbf{B_2}$)
be the list of variables appearing only in $\phi$ ($\psi$). The
combination $\phi \otimes \psi$ is given by
%\begin{equation}
{\footnotesize
\begin{multline}
\phi\otimes \psi (E_1=\alpha_1,\ldots,E_k=\alpha_k,\mathbf{A,B_1,B_2})=\\\sum_{\alpha_{11} \vee \alpha_{12}=\alpha_1}\ldots\sum_{\alpha_{k1} \vee \alpha_{k2}=\alpha_k}
\phi(E_1=\alpha_{11},\ldots,E_k=\alpha_{k1},\mathbf{A,B_1})\psi(E_1=\alpha_{12},\ldots,E_k=\alpha_{k2},\mathbf{A,B_2})
\end{multline}}
%\end{equation}
Factors containing convergent variables are called \textit{heterogeneous} while 
the remaining factors are called \textit{homogeneous}. Heterogeneous factors sharing convergent variables must
be combined with $\otimes$ that we call \textit{heterogeneous multiplication}.

Algorithm {\sc ve1}  exploits causal independence by keeping two lists of factors instead of one:
a list of homogeneous factors $\mathcal{F}_1$ and a list of heterogeneous factors $\mathcal{F}_2$. Procedure {\sc sum-out} is
replaced by {\sc sum-out1} that takes as input  $\mathcal{F}_1$ and  $\mathcal{F}_2$ and a variable $Z$ to be eliminated.
First, all the factors containing $Z$ are removed from  $\mathcal{F}_1$  and combined with multiplication to obtain 
factor $\phi$. Then all the factors  containing $Z$ are removed from  $\mathcal{F}_2$ and combined with heterogeneous 
multiplication obtaining $\psi$. If there are no such factors set $\psi=nil$. In the latter case, {\sc sum-out1} adds the new (homogeneous) factor
$\sum_z \phi$ to $\mathcal{F}_1$ otherwise it adds the new (heterogeneous) factor $\sum_z \phi\psi$ to $\mathcal{F}_2$.
Procedure {\sc ve1} is the same as {\sc ve} with {\sc sum-out} replaced by {\sc sum-out1} and with the difference that
two sets of factors are maintained instead of one.

The $\otimes$ operator   assumes that the convergent variables are
independent given the regular variables. This can be ensured by
\textit{deputising} the convergent variables: every such variable $E$ is
replaced by  a new 
convergent variable $E'$ (called a \textit{deputy variable}), $E'$
that replaces $E'$ in the heterogeneous factors containing $E$, $E$
becomes a regular variable, and a new factor $\iota(E,E')$ is introduced, called \emph{deputy factor}, that represents the identity function
between $E$ and $E'$, i.e., it is defined by

{\small{\centering\begin{tabular}{|l|r|r|r|r|}
\cline{1-5}
$\iota(E, E')$& \textit{ff}& \textit{ft}& \textit{tf}& \textit{tt}\\\cline{1-5}
& 1.0 & 0.0 & 0.0& 1.0  \\\cline{1-5}
\end{tabular}}}

Deputising ensures that we do not have descendents of a convergent
variable in an heterogeneous factor as long as the elimination order
for {\sc ve1} is such that $\rho(E') < \rho(E)$.

 %In general causes $C_1\ldots C_m$ are said to be \textit{causally independent }w.r.t. effect $e$ if there exist
%random variables  $\Xi_1\ldots\Xi_m$ that have the same set of possible values, as $e$
%such that (1) for each $i$,  i probabilistically depends on ci and is conditionally independent of all other cj ’s
%and all other  j ’s given ci, and
%2. There exists a commutative and associative binary operator   over the frame of e such that
%e =  1  2  : :

\subsection{GC-FOVE}
Work on lifting {\sc ve}  started
with \cite{Poole:2003}, and the current state of the art is the algorithm GC-FOVE \cite{DBLP:journals/jair/TaghipourFDB13}, which redefines the operations of C-FOVE~\cite{DBLP:conf/aaai/MilchZKHK08}.
%to be correct for whatever constraint representation
% is being used.
 %This decoupling of the lifted inference algorithm from the constraint representation mechanism allows 
%any constraint language that is closed under these operators to be plugged into the algorithm to obtain an inference system. 
The lifted {\sc ve} algorithm of \cite{DBLP:conf/ilp/GomesC12} represents the adaptation of GC-FOVE to the PFL language.
%, with the difference that they
%use a simple tree to represent the constraints instead
%of the GC-FOVE constraint tree.

%C-FOVE (First-Order Variable Elimination) system's constraint language by allowing arbitrary constraints that define the groups of interchangeable variables over which inference is performed. 
\textit{First-order Variable Elimination} (FOVE) \cite{Poole:2003,DBLP:conf/ijcai/BrazAR05} computes the marginal
probability distribution for query random variables (randvars) by repeatedly applying 
operators that are lifted counterparts of {\sc ve}'s operators.
Models are in the form of a set of parfactors that are essentially the same as in PFL.
A parametrized random variable (PRV) $\mathcal{V}$ is of the form $A|C$,
where $A = F(X_1,\ldots,X_n)$ is a non-ground atom and
$C$ is a constraint on logical variables (logvars) $\textbf{X} = \{X_1,\ldots,X_n\}$. Each PRV represents the set of randvars
$\{F(\mathbf{x})|\textbf{x} \in C\}$, where $\mathbf{x}$ is the tuple of constants $(x_1,\ldots,x_n)$. Given a PRV $\mathcal V$,
we use $RV(\mathcal{V})$ to denote the set of randvars it represents. Each ground
atom is associated with one randvar, which can take
any value in $range(F)$.
 
%A PCRV defines a set of counting randvars through its groundings of all variables in $\mathbf{X} \ \{X\}$, whose value is an histogram that counts how many different values of $X$ occur for each value of $P$.
% C-FOVE adds the counting (C) formulas introduced by \cite{DBLP:conf/aaai/MilchZKHK08} to FOVE. 
% A counting formula has the form $\#_{X_i}[P(\mathbf{X})]$, where $X_i \in \mathbf{X}$ is called the counted logvar. A parametrized counting
% randvar (PCRV) is a pair ($\#_{X_i}[P(\mathbf{X})]$,$C$). For each
% instantiation of $\mathbf{X}\backslash\{X_i\}$, it creates a separate counting
% randvar (CRV). The value of this CRV is a histogram of the form $\{(r_1, n_1), (r_2, n_2),\ldots,(r_k, n_k)\}$: given a valuation for $P(\mathbf{X})$, it counts how many different values of $X_i$ occur for each $r \in range(P)$. Counting randvars do not replace  randvars, are simply a means of representing potential functions more compactly by exploiting their internal structure.
% The reformulation of a factor in terms of a counting
% randvar is called counting conversion. In many situations where lifted elimination cannot immediately be applied, counting
% conversion makes it applicable. The conditions of the sum-out operator  state that an atom $A_i$ can only be eliminated from a parfactor $g$ if $A_i$ has all the logvars in $g$. When an atom has fewer logvars than the parfactor, counting conversion modifies the parfactor by replacing another atom $A_j$ by a counting formula. %, which removes this counted logvar from logvar($\cal{A}$), where $\cal{A}$ =$\{A_1,\ldots,A_n\}$.

GC-FOVE tries to eliminate all (non-query) PRVs in a particular order.
To do so, GC-FOVE supports several operators. It first tries
\emph{Lifted Sum-Out}, that excludes a PRV from a parfactor $\phi$ if
the PRV only occurs in $\phi$. Next, \emph{Lifted Multiplication}, that
multiplies two aligned parfactors. Matching variables must be properly
aligned and the new coefficients must be computed taking into account the
number of groundings in $C$. Third, \emph{Lifted Absorption}
eliminates $n$ PRVs that have the same observed value. If the two
operations cannot be applied, a chosen parfactor must be \emph{split}
so that some of its PRVs match another
parfactor. % All of them, except lifted sum-out, can be
% seen as operators enabling lifted sum-out.
% %: when it can't be applied, the other
% %operators are applied until it can applied. 
% When a particular PRV needs to be eliminated,
% GC-FOVE checks whether the preconditions for lifted
% sum-out hold. If not, C-FOVE applies one or
% more enabling operators until the preconditions are
% satisfied, then applies lifted sum-out.
In the worst case, when none of the lifted operators can be applied, GC-FOVE
resorts to propositionalization: it completely
grounds the  parametrized 
randvars and parfactors and performs inference
on the ground level. 

GC-FOVE further extends PRVs with counting formulas, introduced in
C-FOVE~\cite{DBLP:conf/aaai/MilchZKHK08}. A counting formula takes
advantage of symmetry existing in factors that are products of
independent variables. It represents a factor of the form
$\phi(F(x_1), F(x_2), \ldots, \\F(x_n))$, where all variables have the
same domain, as $\phi(\#_X[F(X)])$. The factor implements a
multinomial distribution, such that its values depend on the number of
variables $n$ and domain size. The lifted counted variable is named a
PCRV. PCRVs may result from summing-out, when we obtain factors with a
single PRV, or through \emph{Counting Conversion} that searches for
factors of the form $\phi(\prod_i(S(X_j)F(x_j,y_i)))$ and counts on
the occurrences of $Y$. Definitions for counting formulas are reported in \ref{app_definition}.

%The operators for multiplication, elimination, and
%counting conversion as defined in C-FOVE are the same in the generalized version, while the constraint
%manipulation operators, that are simply (in)equality constraints such as ($Friend(X, Y),X \neq ann$), are made more flexible.
% The  operators available in GC-FOVE besides multiplications and sum-out are splitting, shattering, expansion, count normalization and absorption. They are defined in terms of relational algebra, indicate which objects are grouped together, enhancing the application of the enabling operators, and are able to manipulate the arbitrary constraints.
GC-FOVE employs a constraint-tree to represent arbitrary constraints $C$, whereas the PFL simply uses sets of tuples. 
% GC-FOVE also extends PRVs with counting formulae, introduced by \cite{DBLP:conf/aaai/MilchZKHK08} to FOVE.  
%A constraint tree on logvars $\mathbf{X}$ is a tree in which
%each internal  node is labeled with a logvar
%$X \in \mathbf{X}$, each leaf is labeled with a terminal label
%$\top$ and each edge is labeled with
%a (sub-)domain $D(e) \subseteq D(X_i)$. Each path from the root to a leaf  represents the Cartesian product of  the tuples $\{x_1,\ldots,x_n\}$ encountered along the traversed edges, i.e., an arbitrary constraint.
%These trees may be re-ordered to simplify the constraint operations. 
Arbitrary constraints can capture more symmetries in
the data, which potentially offers the ability to perform
more operations at a lifted level.

\section{Translating  ProbLog into PFL}
\label{trans}

In order to translate ProbLog into PFL, let us start from the
conversion of a ProbLog program into a Bayesian network with noisy OR
nodes. Here we adapt the conversion for Logic Programs with Annotated Disjunctions
presented in \cite{VenVer04-ICLP04-IC,DBLP:journals/fuin/MeertSB08} to
the case of ProbLog.  The first step is to generate the grounding of
the ProbLog program. For each atom $A$ in the Herbrand base of the
program, the Bayesian network contains a Boolean random variable with
the same name.  Each probabilistic fact $p::A$ is represented by a
parentless node with the conditional probability table (CPT):

{\small{\centering\begin{tabular}{|c|c|c|}
\cline{1-3}
A& \textit{f}& \textit{t}\\\cline{1-3}
& 1-p & p  \\\cline{1-3}
\end{tabular}}}

\noindent
For each ground rule $R_i=H \leftarrow B_1,\ldots,B_n,not(C_1),\ldots,not(C_m)$ we add to the network a random variable called
$H_i$ that has as parents $B_1,\ldots,B_n,C_1,\ldots,C_m$ and the following CPT:

{\small{\centering\begin{tabular}{|l|r|r|}
\cline{1-3}
$H_i$&  $B_1=t,\ldots,B_n=t,C_1=\textit{f},\ldots,C_m=\textit{f}$& \mbox{all other columns}\\\cline{1-3}
f& 0.0 & 1.0  \\\cline{1-3}
t& 1.0 & 0.0  \\\cline{1-3}
\end{tabular}}}

In practice $H_i$ is the result of the conjunction of random variables representing the atoms in the body.
Then for each ground atom $H$ in the Herbrand base not appearing in a probabilistic fact, we add $H$ to the network with parents all $H_i$ of ground rules with $H$ in the head and with the CPT

{\small{\centering\begin{tabular}{|l|r|r|}
\cline{1-3}
H& $ \mbox{at least one $H_i=t$}$& \mbox{all other columns}\\\cline{1-3}
f& 0.0 & 1.0  \\\cline{1-3}
t& 1.0 & 0.0  \\\cline{1-3}
\end{tabular}}}
representing the result of the disjunction of random variables $H_i$.

Translating ProbLog into PFL allows us to stay in the lifted
(non-ground) program.  

\begin{example}[Translation of a ProbLog program into PFL]
\label{problog2pfl}
The  translation of the ProbLog program of Example \ref{ws_attr_problog} into PFL is
\begin{footnotesize}
\begin{verbatim}
bayes series1, s; identity ; [].
bayes series2, attends(P); identity; [person(P)].
bayes series, series1, series2 ; disjunction; [].
bayes attends1(P), at(P,A); identity; [person(P),attribute(A)].
bayes attends(P), attends1(P); identity; [person(P)].
bayes s; [0.9, 0.1]; [].
bayes at(P,A); [0.7, 0.3] ; [person(P),attribute(A)].

identity([1,0,0,1]).
disjunction([1,0,0,0,
             0,1,1,1]).
\end{verbatim}
\end{footnotesize}
\end{example}

Notice that \verb|series2| and \verb|attends1(P)| can be seen as or-nodes. Thus, after grounding,
factors derived from the second and the fourth parfactor should not be multiplied together but should be
combined with heterogeneous multiplication, as variables \verb|series2| and \verb|attends1(P)| are in
fact convergent variables. % sono series2 e attends1(P) che sono convergenti perche' l'or  di attends(P) per le diverse persone e di at(P,A) per i diversi attributi. 
To do so, we need to identify heterogeneous factors and add deputy
variables and factors. We thus introduce two new types of factors to
PFL, \verb|het| and \verb|deputy|. The first factor is such that its
ground instantiations are heterogeneous factors. The convergent
variables are assumed to be represented by the first atom in the
factor's list of atoms. Lifting identity is straightforward, it
corresponds to two atoms and imposes an identity factor between their
ground instantiations. Since the factor is fixed, it is not indicated.
%Thus the translation of two rules
%\begin{small}
%\begin{verbatim}
%A :- B,C1.
%A :- C,C2.
%\end{verbatim}
%\end{small}
%where \verb|C1| and \verb|C2| are conjunction of constraints (atoms for predicates with only ground certain facts)
%is
%\begin{small}
%\begin{verbatim}
%het A1p, B; identity; [C1]
%het A2p, C; identity; [C2]
%deptuy A1, A1p; [C1].
%deptuy A2, A2p; [C2].
%bayes A, A1, A2; disjunction
%\end{verbatim}
%\end{small}

\begin{example}[Extended PFL program]
The PFL program of  Example \ref{problog2pfl}, extended with the two new factors \verb|het| and \verb|deputy|, becomes:
%\begin{small}
%\begin{verbatim}
%bayes series1, s; identity ; [].
%het series2p, attends(P); identity; [person(P)].
%deputy series2, series2p; [].
%bayes series, series1, series2; disjunction ; [].
%het attends1p(P), at(P.A); identity; [person(P),attribute(A)].
%deputy attends(P), attends1p(P); [person(P)].
%bayes s; [0.9, 0.1]; [].
%bayes at(P,A); [0.7, 0.3] ; [person(P),attribute(A)].
%\end{verbatim}
%\end{small}
\begin{footnotesize}
\begin{verbatim}
het series1p, s; identity ; [].
het series2p, attends(P); identity; [person(P)].
deputy series2, series2p; [].
deputy series1, series1p; [].
bayes series, series1, series2; disjunction ; [].
het attends1p(P), at(P.A); identity; [person(P),attribute(A)].
deputy attends1(P), attends1p(P); [person(P)].
bayes attends(P), attends1(P); identity; [person(P)].
bayes s; [0.9, 0.1]; [].
bayes at(P,A); [0.7, 0.3] ; [person(P),attribute(A)].
\end{verbatim}
\end{footnotesize}
where \verb|series1p|, \verb|series2p| and \verb|attends1p(P)| are the convergent deputy random variables, and 
\verb|series1|, \verb|series2| and \verb|attends1(P)| are their corresponding new regular variables. The fifth Bayesian factor represents the combination of the contribution to \verb|series| of the two rules for it. Causal independence could be applied here as well since the combination is really an OR, but for simplicity we decided to concentrate only on exploiting causal independence for the convergent variables 
represented by the head of rules which is the hard part.
\end{example}

%In order to compute the marginal probabilities of a set of random variables given evidence, PFL supports various inference algorithms, including lifted and grounded versions of Variable Elimination ({\sc ve}) and Belief Propagation (BP). For lifted {\sc ve}, it includes an implementation  of GC-FOVE, discussed in the next section.

\section{Heterogeneous Lifted Multiplication and Summation}
\label{hetop}
GC-FOVE must be modified in order to take into account heterogeneous factors and convergent variables.
The {\sc ve} algorithm must be replaced by {\sc ve1}, i.e., two lists of factors must be maintained, one with
homogeneous and the other with heterogeneous factors. When summing out a variable, first the homogeneous factors must be
combined together with homogeneous lifted multiplication. Then the heterogeneous factors must be combined together with 
heterogeneous lifted multiplication and, finally, the two results must be combined to produce a final factor from which 
the random variable is eliminated.

%OPERATOR 1
Lifted heterogeneous multiplication is defined as
Operator~\ref{het-mul}, considering the case in which the two factors
share convergent random variables. We assume familiarity with set and
relational algebra (e.g.,\ % union $\cup$, intersection $\cap$, difference
% $\setminus$, set partitioning, projection $\pi_X$, attribute renaming
% $\rho$,
join $\bowtie$) while some useful definitions are reported in
\ref{app_definition}.  PRVs must be \emph{count-normalized}, that is,
the corresponding parameters must be scaled to take into account
domain size and number of occurrences in the parfactor. PRVs are then
aligned and the joint domain is computed as the natural join between
the set of constraints. Following standard lifted multiplication, we
assume the same PRV will have a different instance in each grounded
factor. We thus proceed very much as in the grounded case, and for
each case $(a_{11},\ldots,a_{1k},\mathbf{b}_1,\mathbf{b}_2)$ we sum
the potentials obtained by multiplying the
$\phi_1(a_{11},\ldots,a_{1k},\mathbf{b}_1)$ and
$\phi_2(a_{11},\ldots,a_{1k},\mathbf{b}_2)$. Note that although
potentials need not be normalised to sum to $1$ until the end, the
relative counts of $\phi_1$ and $\phi_2$ must be weighed by
considering the number of instances $\phi_2$ for each $\phi_1$.

% Last, we compute the cases for
% $A_i = \mathtt{false}$ and $A_i = \mathtt{true}$, knowing that set of
% PRVs $B_1$ repeats $r_i = ||Y|B_2||$ times, and $B_1$ $r_2 = ||Y|B_1||$ times. {\small [**for Vitor: May you check last sentence? perhaps some indexes aren't correct, symbol $||$ isn't defined nor variable $r_i$.]}
%
\begin{myoperator}
{\footnotesize\begin{tabbing}
==\==\==\==\=\kill
\textbf{Operator} \sc{het-multiply}\\
\textbf{Inputs}:\\
(1) $g_1=\phi_1(\mathcal{A}_1)|C_1$: a parfactor in model $G$ with convergent variables $\mathcal{A}_1=\{A_{11},\ldots,A_{1k}\}$\\% with $j=1,\ldots,k$
(2) $g_2=\phi_2(\mathcal{A}_2)|C_2$: a parfactor in model $G$ with convergent variables $\mathcal{A}_2=\{A_{21},\ldots,A_{2k}\}$\\%with $j=1,\ldots,k$
(3) $\theta = \{\mathbf{X}_1\rightarrow \mathbf{X}_2\}$: an alignment between $g_1$ and $g_2$\\
\textbf{Preconditions}:\\
(1) for $i = 1, 2$: $\mathbf{Y}_i = logvar(\mathcal{A}_i)\setminus \mathbf{X}_i$
is count-normalized w.r.t. $\mathbf{X}_i$ in $C_i$\\
\textbf{Output}: $\phi(\mathcal{A})|C$, such that\\
% removed the \rho, it is confusing because it is also used for elimination order.
(1) $C=C_1\theta\bowtie C_2$\\
(2) $\mathcal{A}=\mathcal{A}_1\theta\cup \mathcal{A}_2$\\
(3) Let $\mathcal{A}$ be $(A_1,\ldots,A_k,\mathcal{B})$ with $A_j=A_{1j}\theta=A_{2j}$ for $j=1,\ldots,k$, $\mathcal{B}$ the set of regular variables\\
%(4) for assignments $\mathbf{a}^{\textit{f}}=(\ldots,a_{k-1},\textit{f},a_{k+1},\ldots)$ and \\
%\> $\mathbf{a}^t=(\ldots,a_{k-1},t,a_{k+1},\ldots)$ to $\mathcal{A}$ with \\
%\>$\mathbf{a}^{\textit{f}}_1 =\pi_{\mathcal{A}_1} (\mathbf{a}^{\textit{f}})=(\ldots,a_{i-1},\textit{f},a_{i+1},\ldots)$, $\mathbf{a}^{\textit{f}}_2 =\pi_{\mathcal{A}_2} (\mathbf{a}^{\textit{f}})=(\ldots,a_{j-1},\textit{f},a_{j+1},\ldots)$,\\
%\>$\mathbf{a}^t_1 =\pi_{\mathcal{A}_1} (\mathbf{a}^t)=(\ldots,a_{i-1},t,a_{i+1},\ldots)$, $\mathbf{a}^t_2 =\pi_{\mathcal{A}_2} (\mathbf{a}^t)=(\ldots,a_{j-1},t,a_{j+1},\ldots)$ let\\
%\>\>$\phi(\mathbf{a}^{\textit{f}})=\phi_1(\mathbf{a}^{\textit{f}}_1)^{1/r_2}\phi_2(\mathbf{a}^{\textit{f}}_2)^{1/r_1}$ \\
%\>\>$\phi(\mathbf{a}^t)=
%\phi_1(\mathbf{a}^{\textit{f}}_1)^{1/r_2}\phi_2(\mathbf{a}^t_2)^{1/r_1}+
%\phi_1(\mathbf{a}^t_1)^{1/r_2}\phi_2(\mathbf{a}^{\textit{f}}_2)^{1/r_1}+
%\phi_1(\mathbf{a}^t_1)^{1/r_2}\phi_2(\mathbf{a}^t_2)^{1/r_1}$ \\
%\>with $r_i=\mbox{\sc{Count}}_{\mathbf{Y}_i|\mathbf{X}_i}(C_i)$\\
(4) for each assignment $\mathbf{a}=(a_1,\ldots,a_{k},\mathbf{b})$ to $\mathcal{A}$ with $\mathbf{b}_{1} =\pi_{\mathcal{A}_1\theta} (\mathbf{b})$, $\mathbf{b}_{2} =\pi_{\mathcal{A}_2} (\mathbf{b})$\\
\>$\phi(a_1,\ldots,a_k,\mathbf{b})=$\\
\>\>$\sum_{a_{11} \vee a_{21}=a_1}\ldots\sum_{a_{1k} \vee a_{2k}=a_k}\phi_1(a_{11},\ldots,a_{1k},\mathbf{b}_{1})^{1/r_2}\phi_2(a_{21},\ldots,a_{2k},\mathbf{b}_{2})^{1/r_1}$ \\
\>with $r_i=\mbox{\sc{Count}}_{\mathbf{Y}_i|\mathbf{X}_i}(C_i)$\\
\textbf{Postcondition}: $G \sim  G\setminus\{g_1,g_2\}\cup\{\mbox{\sc{het-multiply}}(g_1,g_2,\theta)\}$
\end{tabbing}}
\caption{Operator {\sc het-multiply}. 
%We assume that the convergent variables appear before the regular ones.
%This is not restrictive as arguments of a factor can be reordered.
\label{het-mul}}
\end{myoperator}
\begin{example}
Consider the heterogeneous parfactors $g_1=\phi_1(p(X_1))|C_1$ and $g_2=\phi_2(p(X_2), q(X_2,Y_2))|C_2$ and suppose that we want to multiply $g_1$ and $g_2$; $p(X)$ is convergent in $g_1$ and $g_2$; $\{X_1\rightarrow X_2\}$ is an alignment between $g_1$ and $g_2$; 
 $Y_2$ is count-normalized w.r.t.\ $X_2$ in $C_2$;  $r_1=\mbox{\sc{Count}}_{Y_1|X_1}(C_1)=2$ and $\mbox{\sc{Count}}_{Y_2|X_2}(C_2)=3$.
Then $\textsc{het-multiply}(g_1,g_2,$ $\{X_1\rightarrow X_2\})=\phi(p(X_2),q(X_2,Y_2))|C$ with $\phi$ given by:

{\small \centering\begin{tabular}{|l|l|}
\cline{1-2}
&\multicolumn{1}{c|}{$\phi(p(X_2),q(X_2,Y_2))$} \\\cline{1-2}
\textit{ff}& $\phi_1(\textit{f})^{1/3}\phi_2(\textit{f,f})^{1/2}$\\\cline{1-2}
\textit{ft}& $\phi_1(\textit{f})^{1/3}\phi_2(\textit{f,t})^{1/2}$\\\cline{1-2}
\textit{tf}& $\phi_1(\textit{f})^{1/3}\phi_2(\textit{t,f})^{1/2}+\phi_1(\textit{t})^{1/3}\phi_2(\textit{f,f})^{1/2}+\phi_1(\textit{t})^{1/3}\phi_2(\textit{t,f})^{1/2}$\\\cline{1-2}
\textit{tt}& $\phi_1(\textit{f})^{1/3}\phi_2(\textit{t,t})^{1/2}+\phi_1(\textit{t})^{1/3}\phi_2(\textit{f,t})^{1/2}+\phi_1(\textit{t})^{1/3}\phi_2(\textit{t,t})^{1/2}$\\\cline{1-2}
\end{tabular}}
%{\footnotesize 
%$$\begin{array}{|l|l|}
%\hline
%&\phi(p(X_2),q(X_2,Y_2)) \\\hline
%\textit{ff}& \phi_1(\textit{f})^{1/3}\phi_2(\textit{f,f})^{1/2}\\\hline
%\textit{ft}& \phi_1(\textit{f})^{1/3}\phi_2(\textit{f,t})^{1/2}\\\hline
%\textit{tf}& \phi_1(\textit{f})^{1/3}\phi_2(\textit{t,f})^{1/2}+\phi_1(\textit{t})^{1/3}\phi_2(\textit{f,f})^{1/2}+\phi_1(\textit{t})^{1/3}\phi_2(\textit{t,f})^{1/2}\\\hline
%\textit{tt}& \phi_1(\textit{f})^{1/3}\phi_2(\textit{t,t})^{1/2}+\phi_1(\textit{t})^{1/3}\phi_2(\textit{f,t})^{1/2}+\phi_1(\textit{t})^{1/3}\phi_2(\textit{t,t})^{1/2}\\\hline
%\end{array}$$}
\end{example}
\begin{example}
Consider the heterogeneous parfactors $g_1=\phi_1(p(X_1),q(X_1,Y_1))|C_1$ and $g_2=\phi_2(p(X_2),$ $ q(X_2,Y_2))|C_2$ and suppose we want to multiply $g_1$ and $g_2$;  all randvars are convergent;  $\{X_1\rightarrow X_2,Y_1\rightarrow Y_2\}$ is an alignment between $g_1$ and $g_2$ (so $r_1=r_2=1$).
Then $\textsc{het-multiply}(g_1,g_2,\{X_1\rightarrow X_2,Y_1\rightarrow Y_2\})=\phi(p(X_2),q(X_2,Y_2))|C$, with $\phi$ given by

{\small \centering\begin{tabular}{|l|l|}
\cline{1-2}
&\multicolumn{1}{c|}{$\phi(p(X_2),q(X_2,Y_2))$} \\\cline{1-2}
\textit{ff}& $\phi_1(\textit{f,f})\phi_2(\textit{f,f})$\\\cline{1-2}
\textit{ft}& $\phi_1(\textit{f,f})\phi_2(\textit{f,t})+\phi_1(\textit{f,t})\phi_2(\textit{f,f})+\phi_1(\textit{f,t})\phi_2(\textit{f,t})$\\\cline{1-2}
\textit{tf}& $\phi_1(\textit{t,f})\phi_2(\textit{f,f})+\phi_1(\textit{f,f})\phi_2(\textit{t,f})+\phi_1(\textit{t,f})\phi_2(\textit{t,f})$\\\cline{1-2}
\textit{tt}& $\phi_1(\textit{f,f})\phi_2(\textit{t,t})+\phi_1(\textit{f,t})\phi_2(\textit{t,f})+\phi_1(\textit{f,t})\phi_2(\textit{t,t})+$\\
&$\phi_1(\textit{t,f})\phi_2(\textit{f,t})+\phi_1(\textit{t,t})\phi_2(\textit{f,f})+\phi_1(\textit{t,t})\phi_2(\textit{f,t})+$\\
&$\phi_1(\textit{t,f})\phi_2(\textit{t,t})+\phi_1(\textit{t,t})\phi_2(\textit{t,f})+\phi_1(\textit{t,t})\phi_2(\textit{t,t})$\\\cline{1-2}
\end{tabular}}
%{\footnotesize 
%$$\begin{array}{|l|l|}
%\hline
%&\phi(p(X_2),q(X_2,Y_2)) \\\hline
%\textit{ff}& \phi_1(\textit{f,f})\phi_2(\textit{f,f})\\\hline
%\textit{ft}& \phi_1(\textit{f,f})\phi_2(\textit{f,t})+\phi_1(\textit{f,t})\phi_2(\textit{f,f})+\phi_1(\textit{f,t})\phi_2(\textit{f,t})\\\hline
%\textit{tf}& \phi_1(\textit{t,f})\phi_2(\textit{f,f})+\phi_1(\textit{f,f})\phi_2(\textit{t,f})+\phi_1(\textit{t,f})\phi_2(\textit{t,f})\\\hline
%\textit{tt}& \phi_1(\textit{f,f})\phi_2(\textit{t,t})+\phi_1(\textit{f,t})\phi_2(\textit{t,f})+\phi_1(\textit{f,t})\phi_2(\textit{t,t})+\\
%&\phi_1(\textit{t,f})\phi_2(\textit{f,t})+\phi_1(\textit{t,t})\phi_2(\textit{f,f})+\phi_1(\textit{t,t})\phi_2(\textit{f,t})+\\
%&\phi_1(\textit{t,f})\phi_2(\textit{t,t})+\phi_1(\textit{t,t})\phi_2(\textit{t,f})+\phi_1(\textit{t,t})\phi_2(\textit{t,t})\\\hline
%\end{array}$$}
\end{example}

The {\sc sum-out} operator must be modified as well. In fact,
consider the case in which a random variable must be summed out from a heterogeneous factor (i.e.
a factor that contains a convergent variable). Consider for example the factor $\phi(p(X),q(X,Y))|C$ with $C=\{x_1,x_2\}\times \{y_1,y_2\}$ and suppose we want to eliminate the PRV $q(X,Y)$. This factor stands for four ground factors of the form $\phi(p(x_i),q(x_i,y_j))$ for $i,j=1,2$
where $p(x_i)$ is convergent. Given an individual $x_i$, the two factors $\phi(p(x_i),q(x_i,y_1))$ and
$\phi(p(x_i),q(x_i,y_2))$ share a convergent variable and cannot be multiplied together with regular multiplication.
In order to sum out $q(X,Y)$ however we must first combine the two factors with heterogeneous multiplication. 
To avoid generating first the ground factors, we have added 
to GC-FOVE  {\sc het-sum-out} (operator \ref{het-sum-out}) that performs the combination and the elimination of a random variable at the same
time. We provide a correctness proof for this operator in \ref{app_proof}.
%OPERATOR 2
\begin{myoperator}
{\footnotesize
\begin{tabbing}
==\===\=\kill
\textbf{Operator} \sc{het-sum-out}\\
\textbf{Inputs}:\\
(1) $g=\phi(\mathcal{A})|C$: a parfactor in model $G$\\
(2) let $\mathcal{A}=(A_1,\ldots,A_k,A_{k+1},\mathcal{B})$ where $A_1,\ldots,A_k$ are convergent atoms\\
(3) $A_{k+1}$ is the atom to be summed out \\
\textbf{Preconditions}:\\
(1) For all PRVs $\mathcal{V}$, other than $A_{k+1}|C$, in $G$: $RV(\mathcal{V})\cap RV(A_{k+1}|C)=\emptyset$\\
(2) $A_{k+1}$ contains all the logvars $X\in logvar(\mathcal{A})$ for which $\pi_X(C)$ is not singleton\\
%(3) $\mathbf{X}^{excl}=\mathbf{X}\setminus logvar(\mathcal{A}\setminus A_i)$ is count-normalized w.r.t.\\
%\> $\mathbf{X}^{com}=\mathbf{X}\cap logvar(\mathcal{A}\setminus A_i)$ in $C$\\
(3) 
$\mathbf{X}^{excl} = logvar(A_{k+1}) \setminus logvar(\mathcal{A} \setminus A_{k+1})$ is count-normalized w.r.t.\\
\>$\mathbf{X}^{com}= logvar(A_{k+1}) \cap logvar(\mathcal{A} \setminus A_{k+1})$ in $C$\\
%$\mathbf{Y}$ is count-normalized w.r.t. $\mathbf{X}$ in $C$\\
\textbf{Output}: $\phi'(\mathcal{A}')|C'$, such that\\
(1) $\mathcal{A}'=\mathcal{A}\setminus A_{k+1}$\\
(2) $C'=\pi_\mathbf{X}(C)$\\
(3) for each assignment $(\mathbf{a}',\mathbf{b})=(a'_{1},\ldots,a'_k,\mathbf{b})$, to $\mathcal{A'}$ \\
\>$\phi'(\mathbf{a}',\mathbf{b})=$\\
\>$\left( \sum_{\mathbf{a}\leq \mathbf{a}'}\sum_{a_{k+1}\in range(A_{k+1})}\textsc{Mul}(A_{k+1},a_{k+1})\phi(a_1,\ldots,a_k,a_{k+1},\mathbf{b}) \right)^r-$\\
\>\> $-\sum_{\mathbf{a}<\mathbf{a}'}\phi'(a_1,\ldots,a_k,\mathbf{b})$
with\\
\> $r=\textsc{Count}_{\mathbf{X}^{excl}|\mathbf{X}^{com}}(C)$\\
%\> $r=\mbox{\sc{Count}}_{\mathbf{X}^{excl}|\mathbf{X}^{com}}(C_1)$\\
\textbf{Postcondition}: $\mathcal{P}_{G\setminus\{g\}\cup\{\textsc{het-sum-out}(g,(A_1,\ldots,A_k),A_{k+1})\}}=\sum_{RV(A_{k+1})}\mathcal{P}_\mathcal{G}$
\end{tabbing}}
\caption{Operator {\sc het-sum-out}.\label{het-sum-out} 
%We assume that the convergent variables appear
%before the regular ones. This is not restrictive as arguments of a factor can be reordered. 
The order $\leq$ between
truth values is the obvious one and between tuples of truth values is the product order induced by $\leq$ between values, i.e., $(a_1,\ldots,a_k)\leq(a'_1,\ldots,a'_k)$ iff $a_i\leq a'_i$ for $i=1,\ldots,k$ and $\mathbf{a}<\mathbf{a'}$ iff 
$\mathbf{a}\leq\mathbf{a'}$ and $\mathbf{a}'\not\leq\mathbf{a}$.
Function {\sc Mul} is defined in \ref{app_definition}.}
\end{myoperator}
\begin{example}
Consider the heterogeneous parfactor $g=\phi(r,p(X),q(X,Y))|C$ and suppose that we want to sum out $q(X,Y)$, that $r$ and $p(X)$ are convergent, that $Y$ is count-normalized w.r.t. $X$ and that $\mbox{\sc{Count}}_{Y|X}(C)=2$.
Then $\textsc{het-sum-out}(g,(r,p(X)),q(X,Y))=\phi'(r,p(X))|C'$ with $\phi'$ given by

{\small \centering\begin{tabular}{|l|l|}
\cline{1-2}
&\multicolumn{1}{c|}{$\phi'(r, p(X))$} \\\cline{1-2}
\textit{ff}& $(\phi(\textit{f,f,f})+\phi(\textit{f,f,t}))^2$\\\cline{1-2}
\textit{ft}& $(\phi(\textit{f,t,f})+\phi(\textit{f,t,t})+\phi(\textit{f,f,f})+\phi(\textit{f,f,t}))^2-\phi'(f,f)$\\\cline{1-2}
\textit{tf}& $(\phi(\textit{t,f,f})+\phi(\textit{t,f,t})+\phi(\textit{f,f,f})+\phi(\textit{f,f,t}))^2-\phi'(f,f)$\\\cline{1-2}
\textit{tt}& $(\phi(\textit{t,t,f})+\phi(\textit{t,t,t})+\phi(\textit{t,f,f})+\phi(\textit{t,f,t})+
\phi(\textit{f,t,f})+\phi(\textit{f,t,t})+\phi(\textit{f,f,f})+\phi(\textit{f,f,t}))^2$-\\&$-\phi'(\textit{t,f})-\phi'(\textit{f,t})-\phi'(\textit{f,f})$\\\cline{1-2} 
\end{tabular}}
%{\footnotesize 
%$$\begin{array}{|l|l|}
%\hline
%&\phi'(r, p(X)) \\\hline
%\textit{ff}& (\phi(\textit{f,f,f})+\phi(\textit{f,f,t}))^2\\\hline
%\textit{ft}& (\phi(\textit{f,t,f})+\phi(\textit{f,t,t})+\phi(\textit{f,f,f})+\phi(\textit{f,f,t}))^2-\phi'(f,f)\\\hline
%\textit{tf}& (\phi(\textit{t,f,f})+\phi(\textit{t,f,t})+\phi(\textit{f,f,f})+\phi(\textit{f,f,t}))^2-\phi'(f,f)\\\hline
%\textit{tt}& (\phi(\textit{t,t,f})+\phi(\textit{t,t,t})+\phi(\textit{t,f,f})+\phi(\textit{t,f,t})+
%\phi(\textit{f,t,f})+\phi(\textit{f,t,t})+\phi(\textit{f,f,f})+\phi(\textit{f,f,t}))^2-\\&-\phi'(t,f)-\phi'(f,t)-\phi'(f,f)\\\hline
%\end{array}$$}
\end{example}

\section{Experiments}
\label{exp}
In order to evaluate the performance of LP$^2$ algorithm, we compare it with PITA and ProbLog2 in two problems:
\emph{workshops attributes} \cite{DBLP:conf/aaai/MilchZKHK08} and Example 7 in \cite{DBLP:conf/ilp/Poole08} that we call \emph{plates}. 
Code of all the problems can be found in \ref{app_programs}.
Moreover, we did a scalability test on a third problem: \emph{competing workshops} \cite{DBLP:conf/aaai/MilchZKHK08}.
All the tests were done on a machine with an Intel Dual Core E6550 2.33GHz processor and 4GB of main memory.
The \emph{workshops attributes}  problem differs from Example \ref{ws_attr_problog} because the first clause for 
\verb|series| is missing and the second clause contains a probabilistic atom in its body.
%The distribution is defined by $m+1$ distinct parfactors: $\forall P.\phi_0 (attends(P),series),$ $\forall P.\phi_1 (attends(P),attr_1 ),...,$ $\forall P.\phi_n (attends(P),attr_m)$.
The \emph{competing workshops} problem differs from \emph{workshops attributes} because it considers, instead of workshop attributes, a set of competing workshops $W$ each one associated with a binary random variable
\emph{hot(W)}, which indicates whether it is focusing on popular research areas. 
 %The distribution is defined by two parfactors.
 %The former is a parfactor $\forall P,W.\phi_1 (attends(P),hot(W))$ that has small potential values
%and models the fact that people are not likely to attend the new workshop if there are popular competing workshops. 
%The latter is a parfactor $\forall P.\phi_2(attends(P),series)$ with larger values, indicating that it is more likely to start a series when
%more people attend the workshop.
The \emph{plates} problem is an artificial example  which contains two sets of individuals, $X$ and $Y$. 
The distribution is defined by 7 probabilistic facts and 9 rules.

Figure \ref{watt-compare} shows the runtime of LP$^2$, PITA and ProbLog2 on the \emph{workshops attributes} problem for the query \texttt{series} where we fixed the number of people to 50 and we increased the number of attributes $m$.
As expected, LP$^2$ is able to solve a much larger set of problems than PITA and ProbLog2.
Figure \ref{watt-ns} shows the time spent by LP$^2$ with up to $10^5$  attributes. 
\begin{figure}
\begin{center}
\subfigure[\label{watt-compare}\textit{Comparison.}]{\includegraphics[scale=0.425]{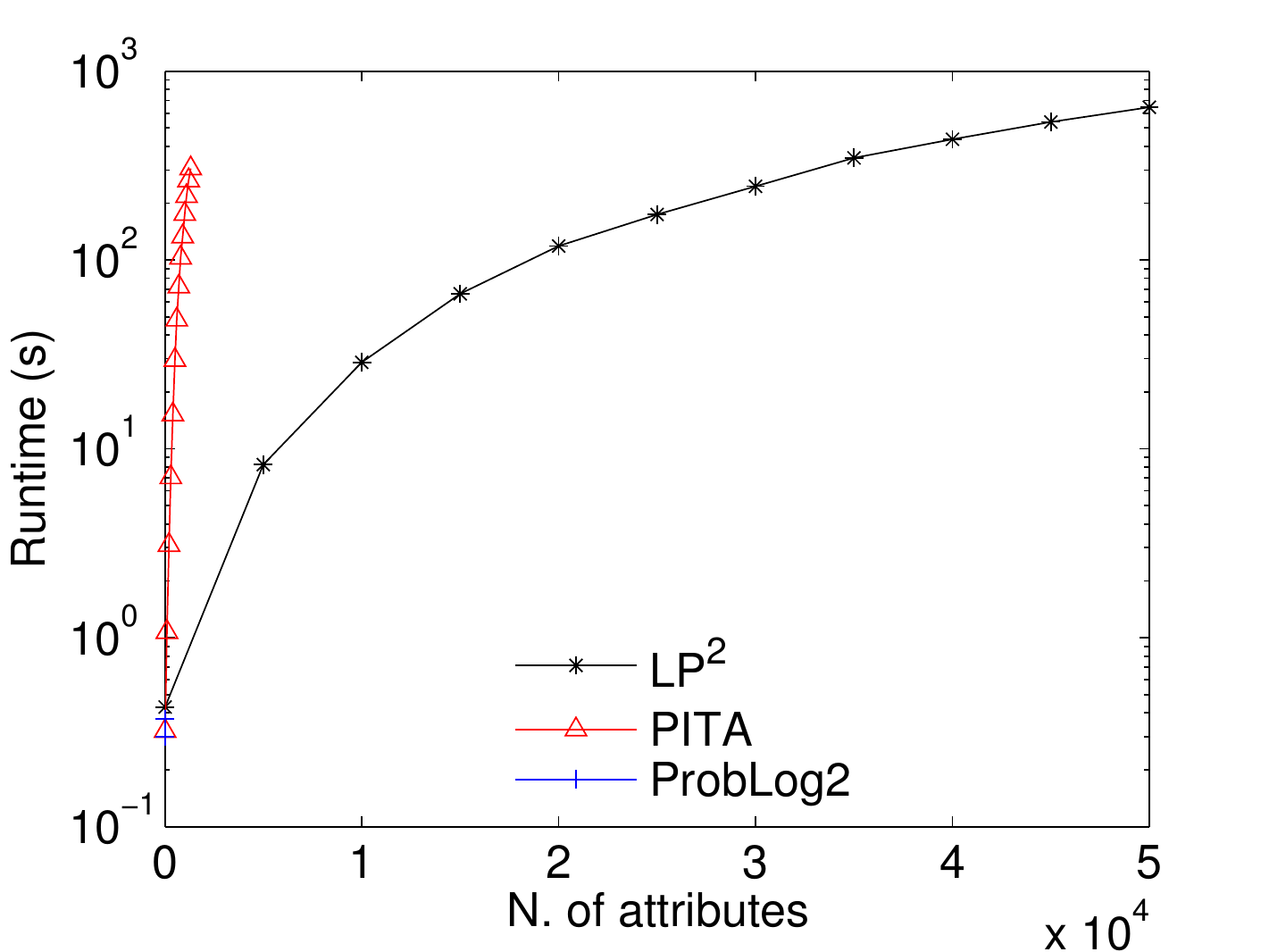}}
\subfigure[\label{watt-ns}\textit{LP$^2$ with up to $10^5$ attributes.}]{\includegraphics[scale=0.425]{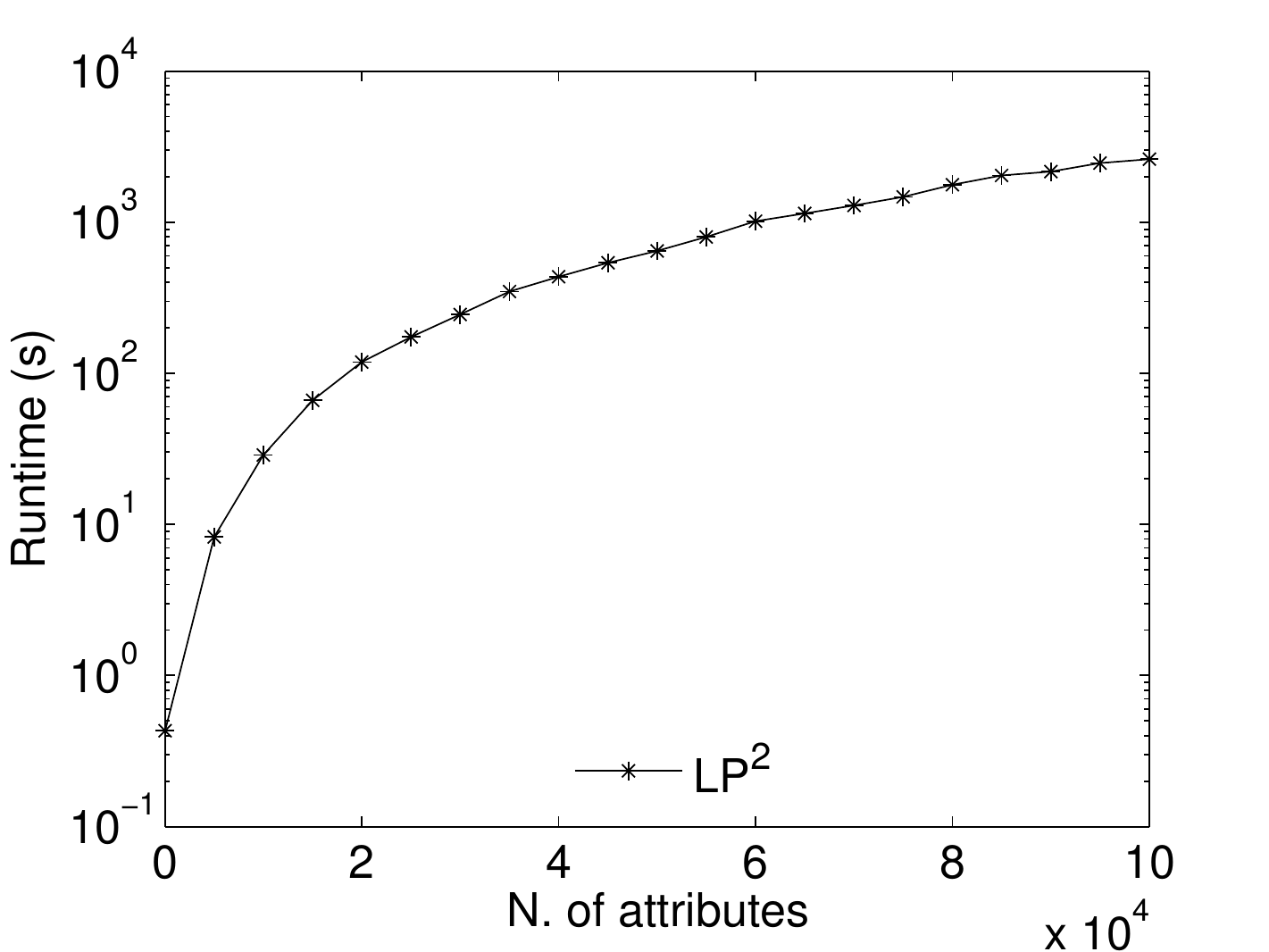}}
\end{center}
\caption{Runtime of LP$^2$, PITA and ProbLog2 on the \textit{workshops attributes}  problems. The Y-axis (runtime) is drawn in log scale. Note that, with the number of people fixed to 50, the problem is intractable by ProbLog2 which can manage at most two attributes.} 
\end{figure} 
% \begin{figure}
% \begin{center}
% \includegraphics[scale=0.425]{./img/watt_ns.pdf}
% \caption{Runtime of LP$^2$ in the workshops attributes problem with up to $10^5$ attributes.} 
% \label{watt-ns}
% \end{center}
% \end{figure} 
Figure \ref{dpoole-pita-problog} shows the runtime of LP$^2$, PITA and ProbLog2 on the \emph{plates} problem, while figure \ref{dpoole-ns} shows that of LP$^2$ with up to $12\times10^4$ $Y$ individuals. 
For this test we executed the query \texttt{f} and we fixed the number of different $X$ individuals to 5 and we increased the number of $Y$ individuals.

\begin{figure*}
\begin{center}
\subfigure
[\label{dpoole-pita-problog}\textit{Comparison.}]{\includegraphics[scale=0.390]{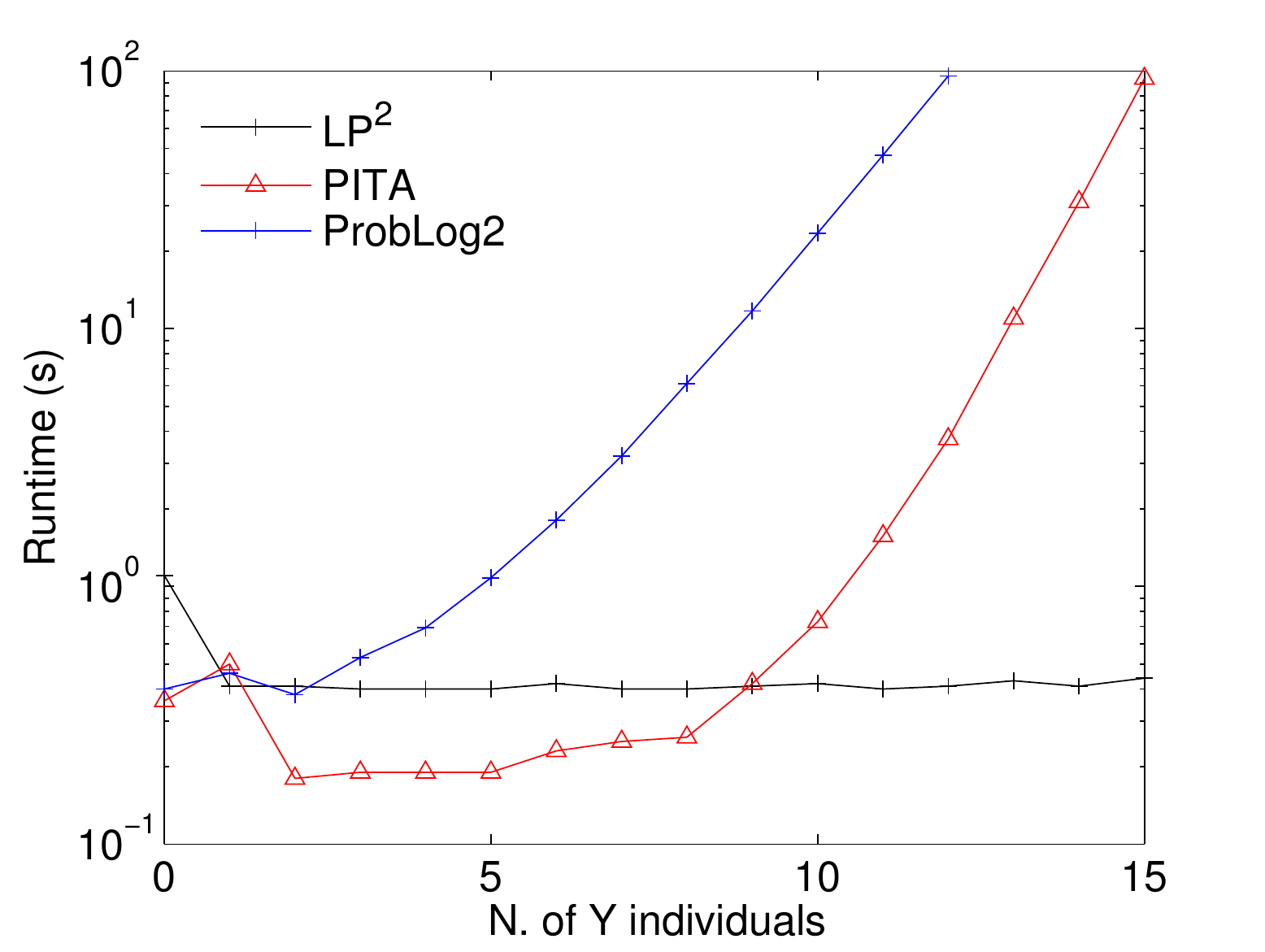}}
\subfigure
[\label{dpoole-ns}\textit{LP$^2$ with up to $12\times 10^4$ Y individuals.}]{\includegraphics[scale=0.425]{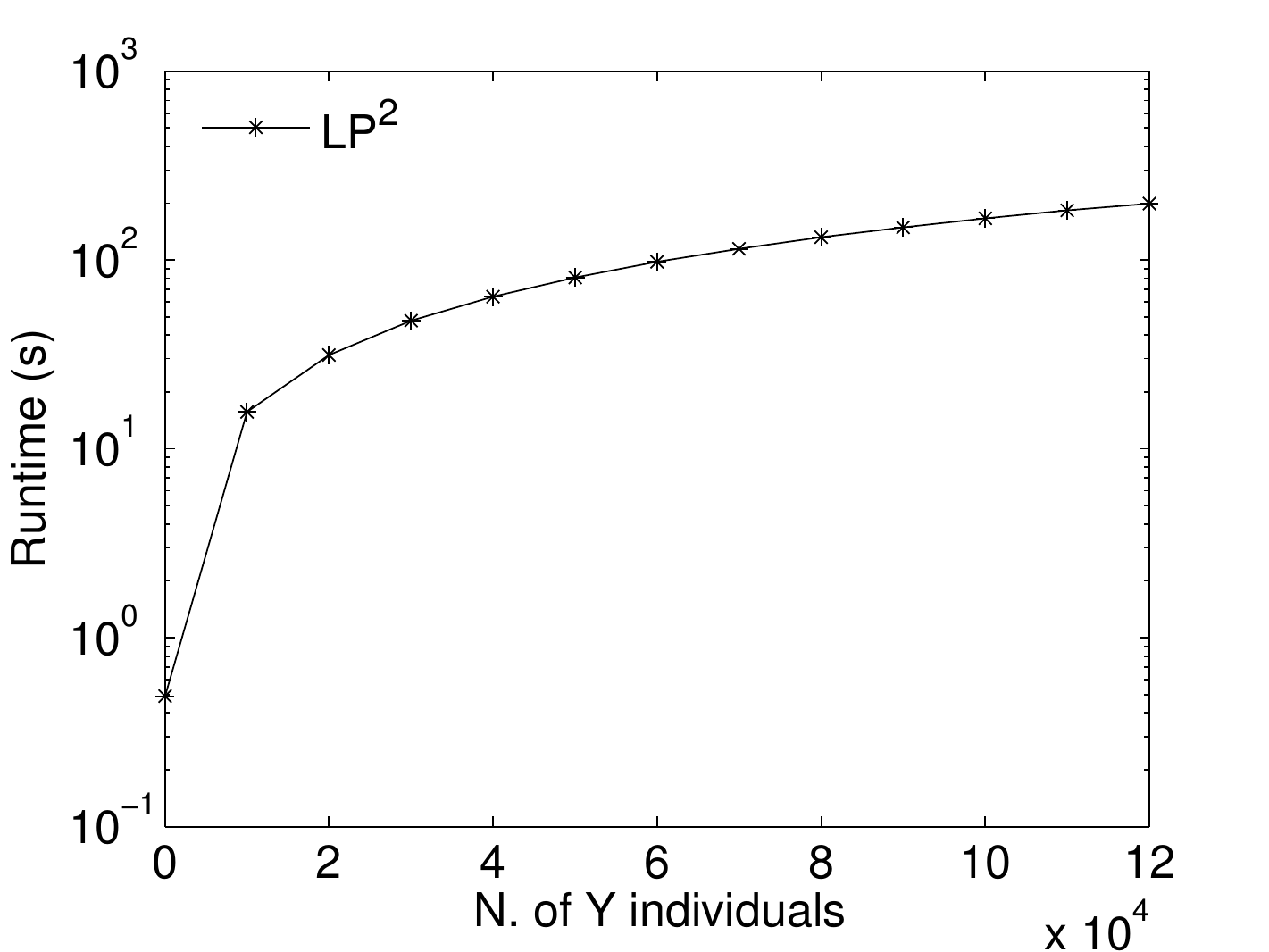}}
\end{center}
\caption{Performance on \textit{plates} with an increasing number of $Y$ individuals. The Y-axis (runtime) is drawn in log scale.} 
\end{figure*} 

Finally, we used the \emph{competing workshops} problem for testing the scalability of LP$^2$. The trend was calculated performing the query \texttt{series} with 10 competing workshops and an increasing number $n$ of people problems.
Figure \ref{wcomp-ns} shows  LP$^2$ computation time. The trend is almost linear in the number of people contained in the problem.

\begin{figure*}
\begin{center}
\includegraphics[scale=0.425]{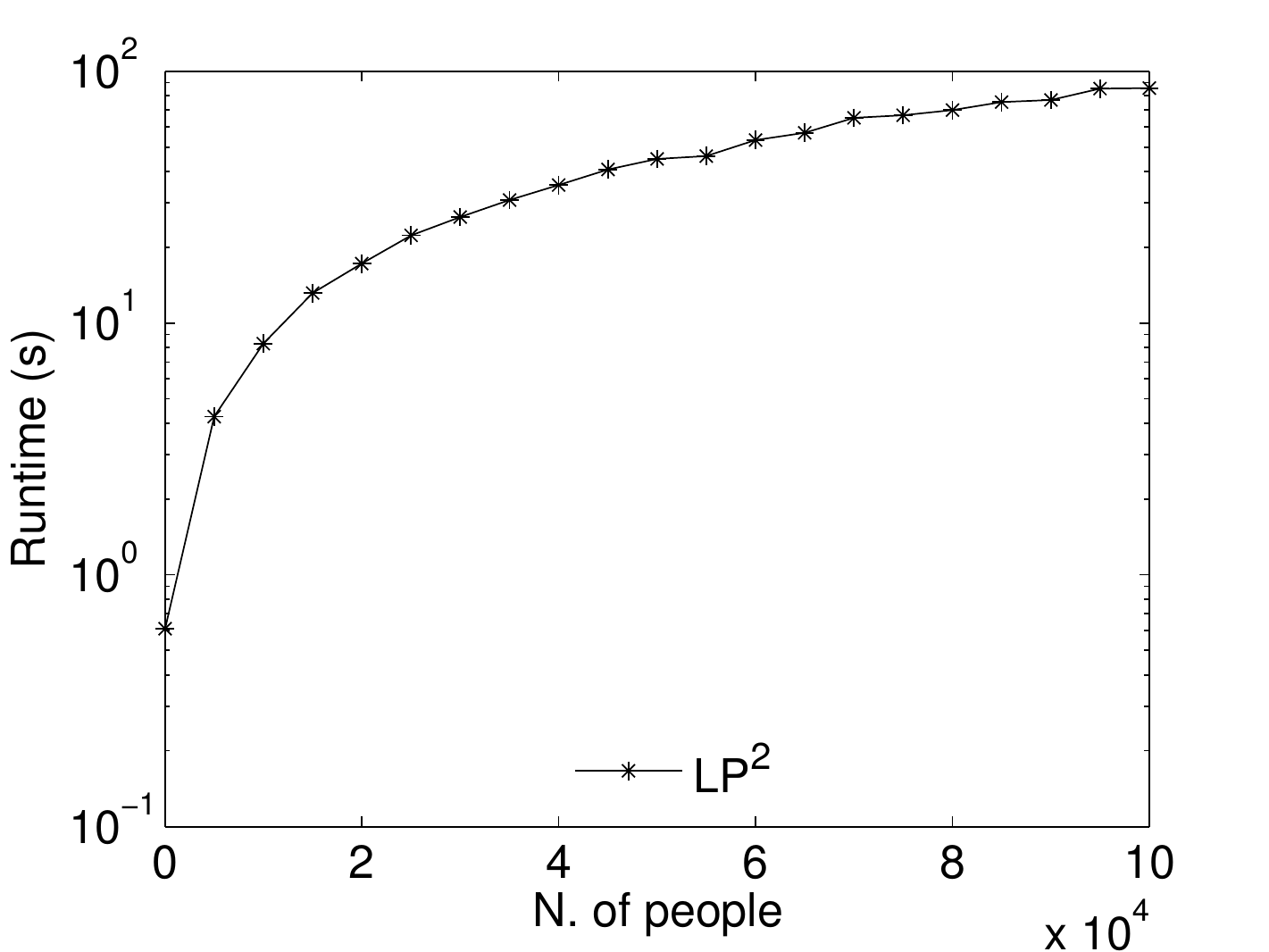}
\end{center}
\caption{\label{wcomp-ns}LP$^2$ runtime in the  \textit{competing workshops} problem. The Y-axis (runtime) is drawn in log scale.}
\end{figure*}

As the results show, LP$^2$ can manage domains that are of several
orders of magnitude larger than the ones managed by PITA and ProbLog2
in a shorter time. %Moreover, it does it in a much smaller time.

\section{Conclusions}
\label{conc}
We have shown that the Lifted Variable Elimination approach is very effective at resolving queries w.r.t.\ probabilistic logic programs containing large amount of facts.
We have proved that with the introduction of the two new heterogeneous operators we can compute the probability of queries 
following the distribution semantics in a very efficient way.
Experimental evidence shows that LP$^2$ can achieve several orders of
magnitude improvements, both in runtime and in the number of facts
that can be managed effectively. 
In the future, we plan to compare our approach with that of \cite{DBLP:conf/ijcai/KisynskiP09,DBLP:conf/uai/TakikawaD99,DBLP:journals/ijis/DiezG03} for dealing with noisy or factors, and to compare with weighted first order model counting~\cite{conf/kr/broeck14}.

\textbf{Acknowledgments:} VSC was partially ﬁnanced by the North
Portugal Regional Operational Programme (ON.2 – O Novo Norte), under
the NSRF, through the ERDF and the Funda\c c\~ao para a Ci\^encia e a
Tecnologia within project ADE/PTDC/EIA-EIA/121686/2010.
This work was supported by "National Group of Computing Science (GNCS-INDAM)".

\bibliographystyle{acmtrans}

\begin{thebibliography}{}

\bibitem[\protect\citeauthoryear{{De Raedt}, Frasconi, Kersting, and
  Muggleton}{{De Raedt} et~al\mbox{.}}{2008}]{DBLP:conf/ilp/2008p}
{\sc {De Raedt}, L.}, {\sc Frasconi, P.}, {\sc Kersting, K.}, {\sc and} {\sc
  Muggleton, S.}, Eds. 2008.
\newblock {\em Probabilistic Inductive Logic Programming - Theory and
  Applications}. LNCS, vol. 4911. Springer.

\bibitem[\protect\citeauthoryear{{De Raedt}, Kimmig, and Toivonen}{{De Raedt}
  et~al\mbox{.}}{2007}]{DBLP:conf/ijcai/RaedtKT07}
{\sc {De Raedt}, L.}, {\sc Kimmig, A.}, {\sc and} {\sc Toivonen, H.} 2007.
\newblock {ProbLog}: A {Probabilistic Prolog} and its application in link
  discovery.
\newblock In {\em 20th International Joint Conference on Artificial
  Intelligence (IJCAI-2007)}. AAAI Press, 2462--2467.

\bibitem[\protect\citeauthoryear{de~Salvo~Braz, Amir, and Roth}{de~Salvo~Braz
  et~al\mbox{.}}{2005}]{DBLP:conf/ijcai/BrazAR05}
{\sc de~Salvo~Braz, R.}, {\sc Amir, E.}, {\sc and} {\sc Roth, D.} 2005.
\newblock Lifted first-order probabilistic inference.
\newblock In {\em 19th International Joint Conference on Artificial
  Intelligence}, {L.~P. Kaelbling} {and} {A.~Saffiotti}, Eds. Professional Book
  Center, 1319--1325.

\bibitem[\protect\citeauthoryear{D{\'i}ez and Gal{\'a}n}{D{\'i}ez and
  Gal{\'a}n}{2003}]{DBLP:journals/ijis/DiezG03}
{\sc D{\'i}ez, F.~J.} {\sc and} {\sc Gal{\'a}n, S.~F.} 2003.
\newblock Efficient computation for the noisy max.
\newblock {\em International Journal of Intelligent Systems\/}, 165--177.

\bibitem[\protect\citeauthoryear{Fierens, Van~den Broeck, Renkens, Shterionov,
  Gutmann, Thon, Janssens, and {De Raedt}}{Fierens
  et~al\mbox{.}}{2014}]{problog2}
{\sc Fierens, D.}, {\sc Van~den Broeck, G.}, {\sc Renkens, J.}, {\sc
  Shterionov, D.}, {\sc Gutmann, B.}, {\sc Thon, I.}, {\sc Janssens, G.}, {\sc
  and} {\sc {De Raedt}, L.} 2014.
\newblock Inference and learning in probabilistic logic programs using weighted
  boolean formulas.
\newblock {\em Theory and Practice of Logic Programming\/}~{\em FirstView
  Articles}.

\bibitem[\protect\citeauthoryear{Getoor and Taskar}{Getoor and
  Taskar}{2007}]{Getoor+al:book07}
{\sc Getoor, L.} {\sc and} {\sc Taskar, B.}, Eds. 2007.
\newblock {\em Introduction to Statistical Relational Learning}.
\newblock MIT Press.

\bibitem[\protect\citeauthoryear{Gomes and Costa}{Gomes and
  Costa}{2012}]{DBLP:conf/ilp/GomesC12}
{\sc Gomes, T.} {\sc and} {\sc Costa, V.~S.} 2012.
\newblock Evaluating inference algorithms for the prolog factor language.
\newblock In {\em 22nd International Conference on Inductive Logic
  Programming}, {F.~Riguzzi} {and} {F.~Zelezn{\'y}}, Eds. LNCS, vol. 7842.
  Springer, 74--85.

\bibitem[\protect\citeauthoryear{Kisynski and Poole}{Kisynski and
  Poole}{2009a}]{DBLP:conf/uai/KisynskiP09}
{\sc Kisynski, J.} {\sc and} {\sc Poole, D.} 2009a.
\newblock Constraint processing in lifted probabilistic inference.
\newblock In {\em 25th Conference on Uncertainty in Artificial Intelligence},
  {J.~Bilmes} {and} {A.~Y. Ng}, Eds. AUAI Press, 293--302.

\bibitem[\protect\citeauthoryear{Kisynski and Poole}{Kisynski and
  Poole}{2009b}]{DBLP:conf/ijcai/KisynskiP09}
{\sc Kisynski, J.} {\sc and} {\sc Poole, D.} 2009b.
\newblock Lifted aggregation in directed first-order probabilistic models.
\newblock In {\em 24th International Joint Conference on Artificial
  Intelligence}, {C.~Boutilier}, Ed. 1922--1929.

\bibitem[\protect\citeauthoryear{Meert, Struyf, and Blockeel}{Meert
  et~al\mbox{.}}{2008}]{DBLP:journals/fuin/MeertSB08}
{\sc Meert, W.}, {\sc Struyf, J.}, {\sc and} {\sc Blockeel, H.} 2008.
\newblock Learning ground {CP-Logic} theories by leveraging {Bayesian} network
  learning techniques.
\newblock {\em Fundamenta Informaticae\/}~{\em 89}, 131--160.

\bibitem[\protect\citeauthoryear{Milch, Zettlemoyer, Kersting, Haimes, and
  Kaelbling}{Milch et~al\mbox{.}}{2008}]{DBLP:conf/aaai/MilchZKHK08}
{\sc Milch, B.}, {\sc Zettlemoyer, L.~S.}, {\sc Kersting, K.}, {\sc Haimes,
  M.}, {\sc and} {\sc Kaelbling, L.~P.} 2008.
\newblock Lifted probabilistic inference with counting formulas.
\newblock In {\em 23rd AAAI Conference on Artificial Intelligence}, {D.~Fox}
  {and} {C.~P. Gomes}, Eds. AAAI Press, 1062--1068.

\bibitem[\protect\citeauthoryear{Poole}{Poole}{1993}]{DBLP:journals/ai/Poole93}
{\sc Poole, D.} 1993.
\newblock Probabilistic horn abduction and {Bayesian} networks.
\newblock {\em Artificial Intelligence\/}~{\em 64,\/}~1, 81--129.

\bibitem[\protect\citeauthoryear{Poole}{Poole}{1997}]{Poo97-ArtInt-IJ}
{\sc Poole, D.} 1997.
\newblock The {I}ndependent {C}hoice {L}ogic for modelling multiple agents
  under uncertainty.
\newblock {\em Artificial Intelligence\/}~{\em 94}, 7--56.

\bibitem[\protect\citeauthoryear{Poole}{Poole}{2003}]{Poole:2003}
{\sc Poole, D.} 2003.
\newblock First-order probabilistic inference.
\newblock In {\em 18th International Joint Conference on Artificial
  Intelligence}, {G.~Gottlob} {and} {T.~Walsh}, Eds. Morgan Kaufmann Publishers
  Inc., 985--991.

\bibitem[\protect\citeauthoryear{Poole}{Poole}{2008}]{DBLP:conf/ilp/Poole08}
{\sc Poole, D.} 2008.
\newblock The independent choice logic and beyond.
\newblock In {\em Probabilistic Inductive Logic Programming}, {L.~De~Raedt},
  {P.~Frasconi}, {K.~Kersting}, {and} {S.~Muggleton}, Eds. LNCS, vol. 4911.
  Springer, 222--243.

\bibitem[\protect\citeauthoryear{Riguzzi and Swift}{Riguzzi and
  Swift}{2011}]{RigSwi11-ICLP11-IJ}
{\sc Riguzzi, F.} {\sc and} {\sc Swift, T.} 2011.
\newblock The {PITA} system: Tabling and answer subsumption for reasoning under
  uncertainty.
\newblock {\em Theory and Practice of Logic Programming, International
  Conference on Logic Programming (ICLP) Special Issue\/}~{\em 11}, 433--449.

\bibitem[\protect\citeauthoryear{Sato}{Sato}{1995}]{DBLP:conf/iclp/Sato95}
{\sc Sato, T.} 1995.
\newblock A statistical learning method for logic programs with distribution
  semantics.
\newblock In {\em 12th International Conference on Logic Programming},
  {L.~Sterling}, Ed. MIT Press, 715--729.

\bibitem[\protect\citeauthoryear{Taghipour, Fierens, Davis, and
  Blockeel}{Taghipour et~al\mbox{.}}{2013}]{DBLP:journals/jair/TaghipourFDB13}
{\sc Taghipour, N.}, {\sc Fierens, D.}, {\sc Davis, J.}, {\sc and} {\sc
  Blockeel, H.} 2013.
\newblock Lifted variable elimination: Decoupling the operators from the
  constraint language.
\newblock {\em Journal of Artificial Intelligence Research\/}~{\em 47},
  393--439.

\bibitem[\protect\citeauthoryear{Takikawa and D'Ambrosio}{Takikawa and
  D'Ambrosio}{1999}]{DBLP:conf/uai/TakikawaD99}
{\sc Takikawa, M.} {\sc and} {\sc D'Ambrosio, B.} 1999.
\newblock Multiplicative factorization of noisy-max.
\newblock In {\em 15th Conference on Uncertainty in Artificial Intelligence}.
  622--630.

\bibitem[\protect\citeauthoryear{Van~den Broeck, Meert, and Darwiche}{Van~den
  Broeck et~al\mbox{.}}{2014}]{conf/kr/broeck14}
{\sc Van~den Broeck, G.}, {\sc Meert, W.}, {\sc and} {\sc Darwiche, A.} 2014.
\newblock Skolemization for weighted first-order model counting.
\newblock {\em ArXiv e-prints\/}~1312.5378v2.
\newblock To appear in the 14th International Conference on Principles of
  Knowledge Representation and Reasoning.

\bibitem[\protect\citeauthoryear{Van~den Broeck, Taghipour, Meert, Davis, and
  Raedt}{Van~den Broeck et~al\mbox{.}}{2011}]{DBLP:conf/ijcai/BroeckTMDR11}
{\sc Van~den Broeck, G.}, {\sc Taghipour, N.}, {\sc Meert, W.}, {\sc Davis,
  J.}, {\sc and} {\sc Raedt, L.~D.} 2011.
\newblock Lifted probabilistic inference by first-order knowledge compilation.
\newblock In {\em 21st International Joint Conference on Artificial
  Intelligence}, {T.~Walsh}, Ed. IJCAI/AAAI, 2178--2185.

\bibitem[\protect\citeauthoryear{Vennekens, Verbaeten, and
  Bruynooghe}{Vennekens et~al\mbox{.}}{2004}]{VenVer04-ICLP04-IC}
{\sc Vennekens, J.}, {\sc Verbaeten, S.}, {\sc and} {\sc Bruynooghe, M.} 2004.
\newblock {Logic Programs With Annotated Disjunctions}.
\newblock In {\em 20th International Conference on Logic Programming},
  {B.~Demoen} {and} {V.~Lifschitz}, Eds. Springer, LNCS 3131, 195--209.

\bibitem[\protect\citeauthoryear{Zhang and Poole}{Zhang and
  Poole}{1996}]{DBLP:journals/jair/ZhangP96}
{\sc Zhang, N.~L.} {\sc and} {\sc Poole, D.~L.} 1996.
\newblock Exploiting causal independence in bayesian network inference.
\newblock {\em Journal of Artificial Intelligence Research\/}~{\em 5},
  301--328.

\end{thebibliography}

\appendix
\section{Problems code}
\label{app_programs}
In this section we present the Problog and PFL code of the testing problems.

\subsection{Workshops Attributes}
To all programs of this section we added 50 workshops and an increasing number of attributes.

\paragraph{Problog program}
\begin{footnotesize}
\begin{verbatim}
series:- person(P),attends(P),sa(P).

0.501::sa(P):-person(P).

attends(P):- person(P),attr(A),at(P,A).

0.3::at(P,A):-person(P),attr(A).
\end{verbatim}
\end{footnotesize}

\paragraph{PFL program}
\begin{footnotesize}
\begin{verbatim}
het series1,ch1(P);[1.0, 0.0, 0.0, 1.0];[person(P)].

deputy series,series1;[].

bayes ch1(P),attends(P),sa(P);[1.0,1.0,1.0,0.0,
                               0.0,0.0,0.0,1.0];[person(P)].
                               
bayes sa(P);[0.499,0.501];[person(P)].

het attends1(P),at(P,A);[1.0, 0.0, 0.0, 1.0];[person(P),attr(A)].

deputy attends(P),attends1(P);[person(P)].

bayes at(P,A);[0.7,0.3];[person(P),attr(A)].
\end{verbatim}

\end{footnotesize}

\subsection{Competing Workshops}
For the \emph{competing workshops} problem we report only the PFL version. For testing purpose we added to this code 10 workshops and an increasing number of people.

\paragraph{PFL program}
\begin{footnotesize}
\begin{verbatim}
bayes ch1(P),attends(P),sa(P);[1.0,1.0,1.0,0.0,
                               0.0,0.0,0.0,1.0];[person(P)].
                               
het series1,ch1(P);[1.0, 0.0, 0.0, 1.0];[person(P)].

deputy series,series1;[].

bayes sa(P);[0.499,0.501];[person(P)].

het attends1(P),ch2(P,W);[1.0, 0.0, 0.0, 1.0];[person(P),workshop(W)].

deputy attends(P),attends1(P);[person(P)].

bayes ch2(P,W),hot(W),ah(P,W);[1.0,1.0,1.0,0.0,
                               0.0,0.0,0.0,1.0];[person(P),workshop(W)].

bayes ah(P,W);[0.2,0.8];[person(P),workshop(W)].
\end{verbatim}
\end{footnotesize}

\subsection{Plates}
For tha \emph{plates} problem we added  5 individuals for $X$ and an increasing number of individuals for $Y$.

\paragraph{Problog program}
\begin{footnotesize}
\begin{verbatim}
f:- e(Y).

e(Y) :- d(Y),n1(Y).
e(Y) :- y(Y),\+ d(Y),n2(Y).

d(Y):- c(X,Y).

c(X,Y):-b(X),n3(X,Y).
c(X,Y):- x(X),\+ b(X),n4(X,Y).

b(X):- a, n5(X).
b(X):- \+ a,n6(X).

a:- n7.

0.1::n1(Y) :-y(Y).
0.2::n2(Y) :-y(Y).
0.3::n3(X,Y) :- x(X),y(Y).
0.4::n4(X,Y) :- x(X),y(Y).
0.5::n5(X) :-x(X).
0.6::n6(X) :-x(X).
0.7::n7.
\end{verbatim}
\end{footnotesize}

\paragraph{PFL program}
\begin{footnotesize}
\begin{verbatim}
het f1,e(Y);[1.0, 0.0, 0.0, 1.0];[y(Y)].

deputy f,f1;[].

bayes e1(Y),d(Y),n1(Y);[1.0, 1.0, 1.0, 0.0,
                        0.0, 0.0, 0.0, 1.0];[y(Y)].

bayes e2(Y),d(Y),n2(Y);[1.0, 0.0, 1.0, 1.0,
                        0.0, 1.0, 0.0, 0.0];[y(Y)].

bayes e(Y),e1(Y),e2(Y);[1.0, 0.0, 0.0, 0.0,
                        0.0, 1.0, 1.0, 1.0];[y(Y)].

het d1(Y),c(X,Y);[1.0, 0.0, 0.0, 1.0];[x(X),y(Y)].

deputy d(Y),d1(Y);[y(Y)].

bayes c1(X,Y),b(X),n3(X,Y);[1.0, 1.0, 1.0, 0.0,
                            0.0, 0.0, 0.0, 1.0];[x(X),y(Y)].

bayes c2(X,Y),b(X),n4(X,Y);[1.0, 0.0, 1.0, 1.0,
                            0.0, 1.0, 0.0, 0.0];[x(X),y(Y)].

bayes c(X,Y),c1(X,Y),c2(X,Y);[1.0, 0.0, 0.0, 0.0,
                              0.0, 1.0, 1.0, 1.0];[x(X),y(Y)].

bayes b1(X),a,n5(X);[1.0, 1.0, 1.0, 0.0,
                     0.0, 0.0, 0.0, 1.0];[x(X)].

bayes b2(X),a,n6(X);[1.0, 0.0, 1.0, 1.0,
                     0.0, 1.0, 0.0, 0.0];[x(X)].

bayes b(X),b1(X),b2(X);[1.0, 0.0, 0.0, 0.0,
                        0.0, 1.0, 1.0, 1.0];[x(X)].

bayes a,n7;[1.0, 0.0, 0.0, 1.0];[].

bayes n1(Y);[0.9, 0.1];[y(Y)].
bayes n2(Y);[0.8, 0.2];[y(Y)].
bayes n3(X,Y);[0.7, 0.3];[x(X),y(Y)].
bayes n4(X,Y);[0.6, 0.4];[x(X),y(Y)].
bayes n5(X);[0.5, 0.5];[x(X)].
bayes n6(X);[0.4, 0.6];[x(X)].
bayes n7;[0.3, 0.7];[].
\end{verbatim}

\end{footnotesize}

\section{Definitions}
\label{app_definition}

\begin{definition}[counting formula]
A counting formula is a syntactic construct of the form $\#_{X_i \in C}[F(\mathbf{X})]$, where $X_i \in \mathbf{X}$ is
called the counted logvar.

A $grounded$ counting formula is a counting formula in which all arguments of the atom $F(\mathbf{X})$, except for the counted logvar, are constants. It defines a counting randvar (CRV) as follows.
\end{definition}

\begin{definition}[counting randvar]
A parametrized counting randvar (PCRV) is a pair ($\#_{X_i}[F(\mathbf{X})],C)$. For each
instantiation of $\mathbf{X} \setminus X_i$, it creates a separate counting randvar (CRV). The value of this CRV is a histogram,
and it depends deterministically on the values of $F(\mathbf{X})$.
Given a valuation for $F(\mathbf{X})$, it counts how many different values of $X_i$ occur for each $r \in range(F)$. The result is a \textit{histogram} of the form
$\{(r_1, n_1), \ldots, (r_k, n_k)\}$, with $r_i \in range(F)$
and $n_i$ the corresponding count.
\end{definition}

\begin{definition}[multiplicity]
The multiplicity of a histogram $h = \{(r1, n1),\ldots,(r_k, n_k)\}$ is a multinomial coefficient, defined as
$$\mbox{{\sc Mul}}(h) = \frac{n!}{\prod_{i=1}^{k}n_i!}.$$
\end{definition}
As multiplicities should only be taken into account for (P)CRVs, never for regular PRVs,
we define for each PRV $A$ and for each value $v \in range(A): \mbox{{\sc Mul}}(A, v) = 1$ if $A$ is a regular PRV, and $\mbox{{\sc Mul}}(A, v) = \mbox{{\sc Mul}}(v)$ if $A$ is a PCRV. This {\sc Mul} function is identical to \cite{DBLP:conf/aaai/MilchZKHK08}'s {\sc num-assign}.

\begin{definition}[Count function]
%\sc{Count}}_{\mathbf{Y}_i|\mathbf{X}_i}(C_i)
Given a constraint $C_\mathbf{X}$, for any $\mathbf{Y} \subseteq \mathbf{X}$ and $\mathbf{Z} \subseteq \mathbf{X-Y}$, the function
 {\sc{Count}}$_{\mathbf{Y}|\mathbf{X}}: C_\mathbf{X} \rightarrow \mathbb{N}$ is defined as follows:
$$\mbox{{\sc{Count}}}_{\mathbf{Y}|\mathbf{Z}}(t) = |\pi_\mathbf{Y}(\sigma_{\mathbf{Z}=\pi_{\mathbf{Z}}(t)}(C_\mathbf{X}))|$$
That is, for any tuple t, this function tells us how many values for $\mathbf{Y}$ co-occur with t's value
for $\mathbf{Z}$ in the constraint. We define
 {\sc{Count}}$_{\mathbf{Y}|\mathbf{Z}}(t) = 1$ when $\mathbf{Y} = \emptyset$. 
\end{definition}

\begin{definition}[Count-normalized constraint]
For any constraint $C_\mathbf{X}$, $\mathbf{Y} \subseteq \mathbf{X}$ and $\mathbf{Z} \subseteq \mathbf{X-Y}$, $\mathbf{Y}$ is count-normalized w.r.t. $\mathbf{Z}$ in $C_\mathbf{X}$ if and only if
$$\exists n \in \mathbb{N}: \forall t \in C_\mathbf{X}: \mbox{{\sc{Count}}}_{\mathbf{Y}|\mathbf{Z}}(t) = n.$$
When such an $n$ exists, we call it the conditional count of $\mathbf{Y}$ given $\mathbf{Z}$ in $C_\mathbf{X}$, and denote it
{\sc{Count}}$_{\mathbf{Y}|\mathbf{Z}}(C_\mathbf{X})$.
\end{definition}

\begin{definition} [substitution]
A substitution $\theta = \{X_1 \rightarrow t_1, \ldots, X_n \rightarrow t_n\} = \{\mathbf{X} \rightarrow \mathbf{t}\}$
maps each logvar $X_i$ to a term $t_i$, which can be a constant or a logvar. When all $t_i$ are constants, $\theta$ is called a grounding substitution, and when all are different logvars, a
renaming substitution. Applying a substitution $\theta$ to an expression $\alpha$ means replacing each occurrence of $X_i$ in $\alpha$ with $t_i$; the result is denoted $\alpha\theta$.
\end{definition}

\begin{definition}[alignment]
An alignment $\theta$ between two parfactors $g = \phi(\mathcal{A})|C$ and $g'= \phi'(\mathcal{A}')|C'$ is a one-to-one substitution $\{\mathbf{X} \rightarrow \mathbf{X}'\}$, with $\mathbf{X} \subseteq logvar(\mathcal{A})$ and $\mathbf{X'} \subseteq logvar(\mathcal{A}')$,
such that $\rho(\pi_\mathbf{X'}(C)) = \pi_\mathbf{X'}(C')$ (with $\rho$ the attribute renaming operator).
\end{definition}
\noindent An alignment tells the multiplication operator that two atoms in two different parfactors represent the same PRV, so it suffices to include it in the resulting parfactor only once.

\section{Correctness proof for heterogeneous multiplication}
\label{app_proof}
%Recall that a set of parfactors $G$ is a compact way of defining 
%ffjf 2 gr(g)^g 2 Gg and the corresponding probability distribution PG(A) = 1
%
%Further, G   G0 means G and G0 de
%
%ning a set of factors gr(G) =
%
%ne the same probability distribution.

\begin{theorem}
Given a model $(\mathcal{F}_1,\mathcal{F}_2)$, two heterogeneous parfactors $g_1,g_2\in\mathcal{F}_2$ and an alignment $\theta$ between $g_1$ and $g_2$, if 
the preconditions of the {\sc het-multiply} operator are fulfilled  then the postcondition
{\small
$$G\setminus\{g_1,g_2\}\cup\{\textsc{het-multiply}(g_1,g_2,\theta)\}$$ }
holds.
\end{theorem}
\begin{proof}
Immediate from the definition of heterogeneous multiplication.
\end{proof}
\begin{theorem}
Given a model $(\mathcal{F}_1,\mathcal{F}_2)$, a heterogeneous parfactor $g\in\mathcal{F}_2$ and an atom $A_{k+1}$ 
to be summed out, if the preconditions of the {\sc het-sum-out} operator are fulfilled  then the postcondition
{\small
$$\mathcal{P}_{G\setminus\{g\}\cup\{\textsc{het-sum-out}(g,(A_1,\ldots,A_k),A_{k+1})\}}=\sum_{RV(A_{k+1})}\mathcal{P}_\mathcal{G}$$} holds.
\end{theorem}
\begin{proof}
We prove the formula giving $phi'$ in Operator \ref{het-sum-out} by double induction over \linebreak$r=\textsc{Count}_{\mathbf{X}^{excl}|\mathbf{X}^{com}}(C)$ and the number $n$ of values at $t$ in the tuple $(a'_1,\ldots,a'_k)$.
For simplicity we assume that the variable to be summed out, $A_{k+1}$, is not a counting variable, but the same reasoning can be applied for a counting variable.
For $n=0$, $r=1$
{\small
$$\phi'(\textit{f},\ldots,\textit{f},\mathbf{b})=\phi(\textit{f},\ldots,\textit{f},\textit{f},\mathbf{b})+\phi(\textit{f},\ldots,\textit{f},t,\mathbf{b})$$}
so the thesis is proved.
For $n=0$, $r>1$, let us call $\phi'_{r}(a'_1,\ldots,a'_k,\mathbf{b})$ the the value of $\phi'$ for $r$. Let us assume that the formula holds for $r-1$. 
For $r>1$, there is an extra valuation $\mathbf{x}_{excl}$ for $\mathbf{X}_{excl}$ given $\mathbf{X}_{com}$ so there is an extra factor $g''(\mathbf{x}_{excl})$. Eliminating $A_{k+1}$ from $g''$ gives
{\small
$$\phi'_1(\textit{f},\ldots,\textit{f},\mathbf{b})=\phi(\textit{f},\ldots,\textit{f},\textit{f},\mathbf{b}) +\phi(\textit{f},\ldots,\textit{f}, \textit{t},\mathbf{b})$$}
This must be multiplied by $\phi'_{r-1}$ with heterogeneous multiplication as $A_1,\ldots,A_k$ are shared obtaining
{\footnotesize
$$\phi'_r(\textit{f},\ldots,\textit{f},\mathbf{b})=
\sum_{\hat{a}_1 \vee \check{a}_1=\textit{f}}\ldots\sum_{\hat{a}_{k} \vee \check{a}_k=\textit{f}}\phi'_{r-1}(\hat{a}_1,\ldots,\hat{a}_k,\mathbf{b})\times \phi'_1(\check{a}_1,\ldots,\check{a}_k,\textit{f},\mathbf{b})=$$
$$=\phi'_{r-1}(\textit{f},\ldots,\textit{f},\mathbf{b})\times \phi'_1(\check{a}_1,\ldots,\check{a}_k,\textit{f},\mathbf{b})=$$
$$=(\phi(\textit{f},\ldots,\textit{f},\textit{f},\mathbf{b}) +\phi(\textit{f},\ldots,\textit{f}, \textit{t},\mathbf{b}))^{r-1}\times
(\phi(\textit{f},\ldots,\textit{f},\textit{f},\mathbf{b}) +\phi(\textit{f},\ldots,\textit{f}, \textit{t},\mathbf{b}))=$$
$$=(\phi(\textit{f},\ldots,\textit{f},\textit{f},\mathbf{b}) +\phi(\textit{f},\ldots,\textit{f}, \textit{t},\mathbf{b}))^r$$
}
so the thesis is proved.

For $n>0$ of values at $t$ in the tuple $\mathbf{a}'=(a'_1,\ldots,a'_k)$, we assume that the formula holds for $(n-1,r)$ and $(n,r-1)$.
For the defintiion of heterogeneous multiplication
{\footnotesize
\begin{eqnarray*}
\phi'_r(a'_1,\ldots,a'_{k},\mathbf{b})&=&\sum_{\hat{a}_1 \vee \check{a}_1=a'_1}\ldots\sum_{\hat{a}_{k} \vee \check{a}_k=a'_{k}}\phi'_{r-1}(\hat{a}_1,\ldots,\hat{a}_{k},\mathbf{b})\phi'_1(\check{a}_1,\ldots,\check{a}_{k},\mathbf{b})
\end{eqnarray*}
}
By adding and removing 

{\small
$\sum_{\mathbf{a}<\mathbf{a}'}\sum_{\hat{a}_1 \vee \check{a}_1=a_1}\ldots\sum_{\hat{a}_{k} \vee \check{a}_k=a_{k}}\phi'_{r-1}(\hat{a}_1,\ldots,\hat{a}_{k},\mathbf{b})\phi'_1(\check{a}_1,\ldots,\check{a}_{k},\mathbf{b})$}

\noindent we get

{\small
$\phi'_r(a'_1,\ldots,a'_{k},\mathbf{b})=$
\begin{eqnarray*}
&=&\sum_{\hat{a}_1 \vee \check{a}_1=a'_1}\ldots\sum_{\hat{a}_{k} \vee \check{a}_k=a'_{k}}\phi'_{r-1}(\hat{a}_1,\ldots,\hat{a}_{k},\mathbf{b})\phi'_1(\check{a}_1,\ldots,\check{a}_{k},\mathbf{b})+\\
&&+\sum_{\mathbf{a}<\mathbf{a}'}\sum_{\hat{a}_1 \vee \check{a}_1=a_1}\ldots\sum_{\hat{a}_{k} \vee \check{a}_k=a_{k}}\phi'_{r-1}(\hat{a}_1,\ldots,\hat{a}_{k},\mathbf{b})\phi'_1(\check{a}_1,\ldots,\check{a}_{k},\mathbf{b})-\\
&&-\sum_{\mathbf{a}<\mathbf{a}'}\sum_{\hat{a}_1 \vee \check{a}_1=a_1}\ldots\sum_{\hat{a}_{k} \vee \check{a}_k=a_{k}}\phi'_{r-1}(\hat{a}_1,\ldots,\hat{a}_{k},\mathbf{b})\phi'_1(\check{a}_1,\ldots,\check{a}_{k},\mathbf{b})=\\
&=&\sum_{\hat{a}_1 \leq a'_1}\sum_{\check{a}_1\leq a'_1}\ldots\sum_{\hat{a}_{k} \leq a'_k}\sum_{\check{a}_k\leq a'_{k}}\phi'_{r-1}(\hat{a}_1,\ldots,\hat{a}_{k},\mathbf{b})\phi'_1(\check{a}_1,\ldots,\check{a}_{k},\mathbf{b})-\\
&&-\sum_{\mathbf{a}<\mathbf{a}'}\sum_{\hat{a}_1 \vee \check{a}_1=a_1}\ldots\sum_{\hat{a}_{k} \vee \check{a}_k=a_{k}}\phi'_{r-1}(\hat{a}_1,\ldots,\hat{a}_{k},\mathbf{b})\phi'_1(\check{a}_1,\ldots,\check{a}_{k},\mathbf{b})=\\
\end{eqnarray*}
}
For the definition of heterogeneous multiplication we obtain
{\small
$\phi'_r(a'_1,\ldots,a'_{k},\mathbf{b})=$
\begin{eqnarray*}
&=&\left(\sum_{\hat{\mathbf{a}}\leq \mathbf{a}'}\phi'_{r-1}(\hat{a}_1,\ldots,\hat{a}_{k},\mathbf{b})\right)\times\left(\sum_{\check{\mathbf{a}}\leq \mathbf{a}'}\phi'_1(\check{a}_1,\ldots,\check{a}_{k},\mathbf{b})\right)-\sum_{\mathbf{a}<\mathbf{a}'}\phi'_{r}(a_1,\ldots,a_k,\mathbf{b})
\end{eqnarray*}
}

By applying the inductive hypothesis for $r-1$ we get
{\small
$\phi'_r(a'_1,\ldots,a'_{k},\mathbf{b})=$}
{\footnotesize
\begin{eqnarray*}
&=&\left(\sum_{\hat{\mathbf{a}} \leq \mathbf{a}'}\left(\sum_{\mathbf{a}\leq \hat{\mathbf{a}}}
\phi(a_1,\ldots,a_{k},\textit{f},\mathbf{b}) +\phi(a_1,\ldots,a_{k},\textit{t},\mathbf{b})\right)^{r-1}-\sum_{\mathbf{a}<\hat{\mathbf{a}}}\phi'_{r-1}(a_1,\ldots,a_{k},\mathbf{b})\right)\times\\
&&\times\left(\sum_{\check{\mathbf{a}}\leq \mathbf{a}'}\phi'_1(\check{a}_1,\ldots,\check{a}_{k},\mathbf{b})\right)-\sum_{\mathbf{a}<\mathbf{a}'}\phi'_{r}(a_1,\ldots,a_k,\mathbf{b})=\\
&=&\left(\sum_{\hat{\mathbf{a}}\leq \mathbf{a}'}\left(\sum_{\mathbf{a} \leq \hat{\mathbf{a}}}
\phi(a_1,\ldots,a_{k},\textit{f},\mathbf{b}) +\phi(a_1,\ldots,a_{k},\textit{t},\mathbf{b})\right)^{r-1}\right)\times\left(\sum_{\check{\mathbf{a}}\leq \mathbf{a}'}\phi'_1(\check{a}_1,\ldots,\check{a}_{k},\mathbf{b})\right)-\\
&&-\left(\sum_{\hat{\mathbf{a}} \leq \mathbf{a}'}\sum_{\mathbf{a}<\hat{\mathbf{a}}}\phi'_{r-1}(a_1,\ldots,a_{k},\mathbf{b})\right)\times\left(\sum_{\check{\mathbf{a}}\leq \mathbf{a}'}\phi'_1(\check{a}_1,\ldots,\check{a}_{k},\mathbf{b})\right)\\
&&-\sum_{\mathbf{a}<\mathbf{a}'}\phi'_{r}(a_1,\ldots,a_k,\mathbf{b})=\\
&=&\left(\sum_{\mathbf{a}\leq \mathbf{a}'}
\phi(a_1,\ldots,a_{k},\textit{f},\mathbf{b}) +\phi(a_1,\ldots,a_{k},\textit{t},\mathbf{b})\right)^{r-1}\times\left(\sum_{\mathbf{a}\leq \mathbf{a}'}\phi'_1(\check{a}_1,\ldots,\check{a}_{k},\mathbf{b})\right)+\\
&&+\left(\sum_{\hat{\mathbf{a}}< \mathbf{a}'}\left(\sum_{\mathbf{a}\leq \hat{\mathbf{a}}}
\phi(a_1,\ldots,a_{k},\textit{f},\mathbf{b}) +\phi(a_1,\ldots,a_{k},\textit{t},\mathbf{b})\right)^{r-1}\right)\times\left(\sum_{\check{\mathbf{a}}\leq \mathbf{a}'}\phi'_1(\check{a}_1,\ldots,\check{a}_{k},\mathbf{b})\right)-\\
&&-\left(\sum_{\hat{\mathbf{a}} \leq\mathbf{a}'}\sum_{\mathbf{a}<\hat{\mathbf{a}}}\phi'_{r-1}(a_1,\ldots,a_{k},\mathbf{b})\right)\times\left(\sum_{\check{\mathbf{a}}\leq \mathbf{a}'}\phi'_1(\check{a}_1,\ldots,\check{a}_{k},\mathbf{b})\right)-\\
&&-\sum_{\mathbf{a}<\mathbf{a}'}\phi'_{r}(a_1,\ldots,a_k,\mathbf{b})=
\end{eqnarray*}
}

By applying the formula for $r=1$\\
{\small
$\phi'_r(a'_1,\ldots,a'_{k},\mathbf{b})=$}
{\small
\begin{eqnarray*}
&=&\left(\sum_{\mathbf{a}\leq \mathbf{a}'}
\phi(a_1,\ldots,a_{k},\textit{f},\mathbf{b}) +\phi(a_1,\ldots,a_{k},\textit{t},\mathbf{b})\right)^{r-1}\times\left(\sum_{\mathbf{a}\leq \mathbf{a}'}\phi(a_1,\ldots,a_{k},\textit{f},\mathbf{b}) +\phi(a_1,\ldots,a_{k},\textit{t},\mathbf{b})\right)+\\
&&+\left(\sum_{\hat{\mathbf{a}}< \mathbf{a}'}\left(\sum_{\mathbf{a}\leq \hat{\mathbf{a}}}
\phi(a_1,\ldots,a_{k},\textit{f},\mathbf{b}) +\phi(a_1,\ldots,a_{k},\textit{t},\mathbf{b})\right)^{r-1}\right)\times\left(\sum_{\check{\mathbf{a}}\leq \mathbf{a}'}\phi'_1(\check{a}_1,\ldots,\check{a}_{k},\mathbf{b})\right)-\\
&&-\left(\sum_{\hat{\mathbf{a}} <\mathbf{a}'}\sum_{\mathbf{a}<\hat{\mathbf{a}}}\phi'_{r-1}(a_1,\ldots,a_{k},\mathbf{b})\right)\times\left(\sum_{\check{\mathbf{a}}\leq \mathbf{a}'}\phi'_1(\check{a}_1,\ldots,\check{a}_{k},\mathbf{b})\right)-\\
&&-\sum_{\mathbf{a}<\mathbf{a}'}\phi'_{r}(a_1,\ldots,a_k,\mathbf{b})=\\
&=&\left(\sum_{\mathbf{a}\leq \mathbf{a}'}
\phi(a_1,\ldots,a_{k},\textit{f},\mathbf{b}) +\phi(a_1,\ldots,a_{k},\textit{t},\mathbf{b})\right)^{r}-\sum_{\mathbf{a}<\mathbf{a}'}\phi'_{r}(a_1,\ldots,a_k,\mathbf{b})+\\
&&+\left(\sum_{\hat{\mathbf{a}}< \mathbf{a}'}\left(\sum_{\mathbf{a}\leq \hat{\mathbf{a}}}
\phi(a_1,\ldots,a_{k},\textit{f},\mathbf{b}) +\phi(a_1,\ldots,a_{k},\textit{t},\mathbf{b})\right)^{r-1}\right)\times\left(\sum_{\check{\mathbf{a}}\leq \mathbf{a}'}\phi'_1(\check{a}_1,\ldots,\check{a}_{k},\mathbf{b})\right)-\\
&&-\left(\sum_{\hat{\mathbf{a}} <\mathbf{a}'}\sum_{\mathbf{a}<\hat{\mathbf{a}}}\phi'_{r-1}(a_1,\ldots,a_{k},\mathbf{b})\right)\times\left(\sum_{\check{\mathbf{a}}\leq \mathbf{a}'}\phi'_1(\check{a}_1,\ldots,\check{a}_{k},\mathbf{b})\right)\\
\end{eqnarray*}
}

By collecting the factor $\sum_{\check{\mathbf{a}}\leq \mathbf{a}'}\phi'_1(\check{a}_1,\ldots,\check{a}_{k},\mathbf{b})$
{\small
$\phi'_r(a'_1,\ldots,a'_{k},\mathbf{b})=$
\begin{eqnarray*}
&=&\left(\sum_{\mathbf{a}\leq \mathbf{a}'}
\phi(a_1,\ldots,a_{k},\textit{f},\mathbf{b}) +\phi(a_1,\ldots,a_{k},\textit{t},\mathbf{b})\right)^{r}-\sum_{\mathbf{a}<\mathbf{a}'}\phi'_{r}(a_1,\ldots,a_k,\mathbf{b})+\\
&&+\left(\sum_{\hat{\mathbf{a}}< \mathbf{a}'}\left(\sum_{\mathbf{a}\leq \hat{\mathbf{a}}}
\phi(a_1,\ldots,a_{k},\textit{f},\mathbf{b}) +\phi(a_1,\ldots,a_{k},\textit{t},\mathbf{b})\right)^{r-1}-\sum_{\hat{\mathbf{a}}<\mathbf{a}'}\sum_{\mathbf{a}\leq \hat{\mathbf{a}}}\phi'_{r-1}(a_1,\ldots,a_{k},\mathbf{b})\right)\times\\
&&\times\left(\sum_{\check{\mathbf{a}}\leq \mathbf{a}'}\phi'_1(\check{a}_1,\ldots,\check{a}_{k},\mathbf{b})\right)=\\
&=&\left(\sum_{\mathbf{a}\leq \mathbf{a}'}
\phi(a_1,\ldots,a_{k},\textit{f},\mathbf{b}) +\phi(a_1,\ldots,a_{k},\textit{t},\mathbf{b})\right)^{r}-\sum_{\mathbf{a}<\mathbf{a}'}\phi'_{r}(a_1,\ldots,a_k,\mathbf{b})+\\
&&+\left(\sum_{\hat{\mathbf{a}}< \mathbf{a}'}\left(\sum_{\mathbf{a}\leq \hat{\mathbf{a}}}
\phi(a_1,\ldots,a_{k},\textit{f},\mathbf{b}) +\phi(a_1,\ldots,a_{k},\textit{t},\mathbf{b})\right)^{r-1}-\sum_{\hat{\mathbf{a}}<\mathbf{a}'}\sum_{\mathbf{a}< \hat{\mathbf{a}}}\phi'_{r-1}(a_1,\ldots,a_{k},\mathbf{b})-\right.\\
&&\left.-\sum_{\mathbf{a}<\mathbf{a}'}\phi'_{r-1}(a_1,\ldots,a_{k},\mathbf{b})\right)\times(\sum_{\mathbf{a}\leq \mathbf{a}'}\phi'_1(a_1,\ldots,a_{k},\mathbf{b}))=\\
&=&\left(\sum_{\mathbf{a}\leq \mathbf{a}'}
\phi(a_1,\ldots,a_{k},\textit{f},\mathbf{b}) +\phi(a_1,\ldots,a_{k},\textit{t},\mathbf{b})\right)^{r}-\sum_{\mathbf{a}<\mathbf{a}'}\phi'_{r}(a_1,\ldots,a_k,\mathbf{b})+\\
&&+\left(\sum_{\mathbf{a}< \mathbf{a}'}\phi'_{r-1}(a_1,\ldots,a_{k},\mathbf{b})-\sum_{\mathbf{a}<\mathbf{a}'}\phi'_{r-1}(a_1,\ldots,a_{k},\mathbf{b})
\right)\times \sum_{\mathbf{a}\leq \mathbf{a}'}\phi'_1(a_1,\ldots,a_{k},\mathbf{b})=\\
&=&\left(\sum_{\mathbf{a}\leq \mathbf{a}'}
\phi(a_1,\ldots,a_{k},\textit{f},\mathbf{b}) +\phi(a_1,\ldots,a_{k},\textit{t},\mathbf{b})\right)^{r}-\sum_{\mathbf{a}<\mathbf{a}'}\phi'_{r}(a_1,\ldots,a_k,\mathbf{b})
\end{eqnarray*}
}
\end{proof}

\label{lastpage}
\end{document}